\DocumentMetadata 
{
	lang		= en-US,
	pdfversion  = 1.7,
    pdfstandard = a-2b,
}

\documentclass[letterpaper, journal]{IEEEtran} 

\usepackage{amsfonts}
\usepackage{amsmath}
\usepackage{amssymb}
\usepackage{amsthm} 

\usepackage{listings}
\usepackage{prettyref}
\usepackage{xspace}
\usepackage{xcolor}
\usepackage{graphicx}
\usepackage{bm}

\usepackage{multirow}
\usepackage{hhline}
\usepackage{booktabs}
\usepackage{adjustbox}

\usepackage[caption=false,font=normalsize,labelfont=sf,textfont=sf]{subfig}
\usepackage{floatrow}

\usepackage[short]{optidef}

\usepackage[shortlabels]{enumitem}
\usepackage[ruled,linesnumbered]{algorithm2e}

\usepackage{tikz}
\usepackage{tikz-3dplot}
\usetikzlibrary{math}
\usepackage{pgfplots}
\usetikzlibrary{fillbetween}
\usepgfplotslibrary{fillbetween}
\pgfplotsset{compat=1.18}
\pgfdeclarelayer{bg} 
\pgfsetlayers{bg,main}

 \usepackage[letterpaper,
            left=1in,
            right=1in,
            top=1in,
            bottom=1in]{geometry}
\usepackage{hyperref}

\graphicspath{ {figures/}}

\newrefformat{prob}{Problem\,\ref{#1}}
\newrefformat{def}{Definition\,\ref{#1}}
\newrefformat{sec}{Section\,\ref{#1}}
\newrefformat{sub}{Section\,\ref{#1}}
\newrefformat{prop}{Proposition\,\ref{#1}}
\newrefformat{app}{Appendix\,\ref{#1}}
\newrefformat{alg}{Algorithm\,\ref{#1}}
\newrefformat{cor}{Corollary\,\ref{#1}}
\newrefformat{thm}{Theorem\,\ref{#1}}
\newrefformat{lem}{Lemma\,\ref{#1}}
\newrefformat{fig}{Fig.\,\ref{#1}}
\newrefformat{tab}{Table\,\ref{#1}}

\newtheorem{theorem}{Theorem}
\newtheorem{problem}{Problem}

\newtheorem{lemma}[theorem]{Lemma}

\newtheorem{proposition}[theorem]{Proposition}
\newtheorem{remark}[theorem]{Remark}

\newcommand{\bdmath}{\begin{dmath}}
\newcommand{\edmath}{\end{dmath}}
\newcommand{\beq}{\begin{equation}}
\newcommand{\eeq}{\end{equation}}
\newcommand{\bdm}{\begin{displaymath}}
\newcommand{\edm}{\end{displaymath}}
\newcommand{\bea}{\begin{eqnarray}}
\newcommand{\eea}{\end{eqnarray}}
\newcommand{\beal}{\beq \begin{array}{ll}}
\newcommand{\eeal}{\end{array} \eeq}
\newcommand{\beas}{\begin{eqnarray*}}
\newcommand{\eeas}{\end{eqnarray*}}
\newcommand{\ba}{\begin{array}}
\newcommand{\ea}{\end{array}}
\newcommand{\bit}{\begin{itemize}}
\newcommand{\eit}{\end{itemize}}
\newcommand{\ben}{\begin{enumerate}}
\newcommand{\een}{\end{enumerate}}

\newcommand{\eg}{\emph{e.g.,}\xspace}
\newcommand{\ie}{\emph{i.e.,}\xspace}

\newcommand{\hide}[1]{}

\newcommand{\hiddenText}{{\color{gray} hidden text.}}
\newcommand{\hideWithText}[1]{\hiddenText}

\newcommand{\eye}{{\mathbf I}}

\newcommand{\at}[1]{^{(#1)}}

\newcommand{\SO}[1]{\ensuremath{\mathrm{SO}(#1)}\xspace}

\newcommand{\T}{\mathsf{T}}

\newcommand{\blue}[1]{{\color{blue}#1}}

\newcommand{\linkToPdf}[1]{\href{#1}{\blue{(pdf)}}}
\newcommand{\linkToPpt}[1]{\href{#1}{\blue{(ppt)}}}
\newcommand{\linkToCode}[1]{\href{#1}{\blue{(code)}}}
\newcommand{\linkToWeb}[1]{\href{#1}{\blue{(web)}}}
\newcommand{\linkToVideo}[1]{\href{#1}{\blue{(video)}}}
\newcommand{\linkToMedia}[1]{\href{#1}{\blue{(media)}}}
\newcommand{\award}[1]{\xspace}

\newcommand{\eps}{\epsilon}
\newcommand{\RR}{\mathbb{R}}

\newrefformat{appendix}{Appendix~\ref{#1}}
\newrefformat{sec}{Section~\ref{#1}}
\newrefformat{prob}{Problem~\ref{#1}}
\newrefformat{def}{Definition~\ref{#1}}
\newrefformat{prop}{Proposition~\ref{#1}}
\newrefformat{alg}{Algorithm~\ref{#1}}
\newrefformat{cor}{Corollary~\ref{#1}}
\newrefformat{thm}{Therorem~\ref{#1}}
\newrefformat{lem}{Lemma~\ref{#1}}
\newrefformat{fig}{Figure~\ref{#1}}
\newrefformat{tab}{Table~\ref{#1}}

\renewcommand{\linkToPdf}[1]{\href{#1}{}}
\renewcommand{\linkToPpt}[1]{\href{#1}{}}
\renewcommand{\linkToCode}[1]{\href{#1}{}}
\renewcommand{\linkToWeb}[1]{\href{#1}{}}
\renewcommand{\linkToVideo}[1]{\href{#1}{}}
\renewcommand{\linkToMedia}[1]{\href{#1}{}}

\def\*#1{\mathbf{#1}}
\def\'#1{\bm{#1}}

\newcommand{\ez}{\mathbf{\hat{e}}_3}

\title{Uncertainty Quantification for Visual Object Pose Estimation:
\\[-0.6ex] S-Lemma Ellipsoidal Bounds}

\author{Lorenzo Shaikewitz,
    Charis Georgiou, and
    Luca Carlone
\thanks{Manuscript received November 22, 2025; revised April 29, 2026; accepted June 29, 2026. This article was recommended for publication by editor Jeanette Bohg upon evaluation of the reviewers' comments. (\emph{Corresponding author: Lorenzo Shaikewitz}).}\thanks{This work was supported by the AFOSR “Certifiable and Self-Supervised Category-Level Tracking” program, Carlone’s NSF CAREER award, and the ONR RAPID program. L. Shaikewitz is supported by an NSF graduate research fellowship. L. Carlone holds concurrent appointments at MIT and as an Amazon Scholar. This paper describes work performed at MIT and is not associated with Amazon.}\thanks{All authors are with the Laboratory for Information and Decision Systems, Massachusetts Institute of Technology, Cambridge, MA. Emails: {\tt\footnotesize \{lorenzos, cgeo, lcarlone\}@mit.edu}.}
}

\begin{document}

\maketitle

\thispagestyle{empty}
\begin{tikzpicture}[overlay, remember picture]
\node[anchor=north, yshift=-0.5cm, text width=2\textwidth, align=center] at (current page.north) {
    \textbf{This paper has been accepted for publication in \emph{IEEE Transactions on Robotics.}}\\
    Please cite the paper as: Lorenzo Shaikewitz, Charis Georgiou, and Luca Carlone,\\
    ``Uncertainty Quantification for Visual Object Pose Estimation: S-Lemma Ellipsoidal Bounds,''
    \emph{IEEE Trans. Robotics}, 2026.
};
\end{tikzpicture}
\vspace{-0.5cm}

\begin{abstract}
    Quantifying the uncertainty of an object's pose estimate is essential for robust control and planning. Although pose estimation is a well-studied robotics problem, attaching statistically rigorous uncertainty is not well understood without strict distributional assumptions. We develop distribution-free pose uncertainty bounds about a given pose estimate in the monocular setting. Our pose uncertainty only requires high probability noise bounds on pixel detections of 2D semantic \emph{keypoints} on a known object. This noise model induces an implicit, non-convex set of pose uncertainty constraints. Our key contribution is SLUE (S-Lemma Uncertainty Estimation), a convex program to reduce this set to a single ellipsoidal uncertainty bound that is guaranteed to contain the true object pose with high probability. SLUE solves a relaxation of the minimum volume bounding ellipsoid problem inspired by the celebrated S-lemma. It requires no initial guess of the bound's shape or size and is guaranteed to contain the true object pose with high probability. For tighter uncertainty bounds at the same confidence, we extend SLUE to a sum-of-squares relaxation hierarchy which is guaranteed to converge to the minimum volume ellipsoidal uncertainty bound for a given set of keypoint constraints. We show this pose uncertainty bound can easily be projected to independent translation and axis-angle orientation bounds. We evaluate SLUE on two pose estimation datasets and a real-world drone tracking scenario. Compared to prior work, SLUE generates substantially smaller translation bounds and competitive orientation bounds.
\end{abstract} 

\section{Introduction}

\IEEEPARstart{V}{ision}-based object pose estimation is an important problem in robotics, used in planning and control for manipulation~\cite{Wen20iros-seTrack,Li24arxiv-DROP} and navigation~\cite{Peng23-trackingDriving}. As with any point estimate, a rigorous notion of uncertainty is crucial for robust downstream decision-making. An autonomous vehicle must slow down to avoid a traffic cone with uncertain position, but can drive past one clearly on the sidewalk. Among the plethora of pose estimation techniques, there are many uncertainty heuristics. These include ensembles~\cite{Shi21arxiv-FastUQPose}, Bayesian inference (\eg particle filters~\cite{Deng19rss-poserbpf}), and learned uncertainty models~\cite{Brachmann16cpvr-UncertaintyPose, Lin22icra-keypointTracking, Hodan20cvpr-epos, Okorn20iros-OrientationDistributions}. 
Although these approaches carry no statistical guarantees (or require strong assumptions to obtain guarantees), they generally capture uncertainty intrinsic to the problem. In monocular estimation, this includes scale ambiguity and object symmetries. Instead of a heuristic, we develop statistically rigorous pose uncertainty which also captures scale ambiguity (although it does not capture object symmetry). Further, our uncertainty bounds hold under minimal distributional assumptions and are agnostic to the choice of pose estimator.

\begin{figure}[tb]
    \centering
\includegraphics[width=0.99\linewidth]{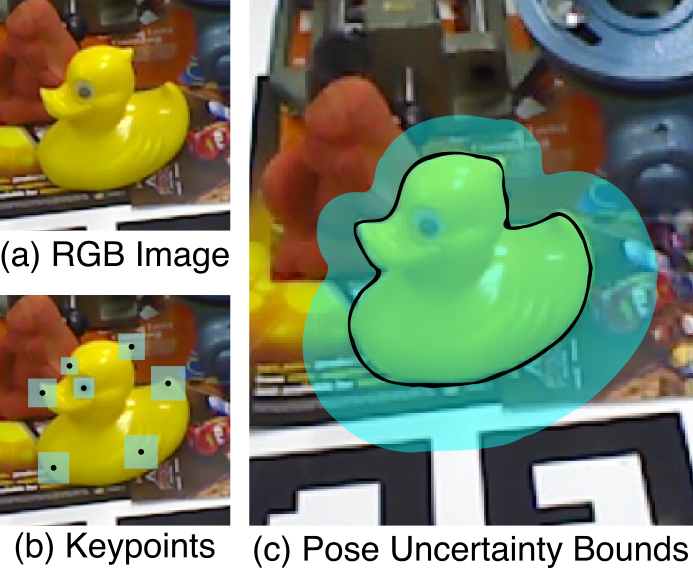}
    \caption{\textbf{Conformal Pose and Uncertainty Estimation.} Given an RGB image of an object (a), we extract 2D semantic keypoints and conformal uncertainty sets (b) which contain the ground truth keypoint with high probability. These sets imply a non-convex set of quadratic constraints on the object pose. We use a generalization of the S-Lemma and a projection scheme to reduce this set to an explicit bound (c) containing the true object pose with high probability and centered at the pose estimate, highlighted in black.}
    \label{fig:intro}
\end{figure}


We phrase the distribution-free uncertainty estimation problem as estimating an explicit uncertainty set which contains the true object pose with high probability. To capture the scale ambiguity inherent to the monocular setting, we begin with high-probability uncertainty sets about a sparse set of RGB \emph{keypoint} measurements (obtained with, \eg conformal prediction) and propagate keypoint uncertainty to pose uncertainty. Following~\cite{Yang23arxiv-ransag}, the uncertainty estimation problem lies in bounding this non-convex pose constraint set with an interpretable outer bound. Prior work simplifies this problem by fixing the shape of the outer bound, for example to a spherical bound~\cite{Yang23arxiv-ransag,Tang24l4dc-setMembership}. We remove this simplification and directly solve for minimum-volume ellipsoidal outer bounds on pose uncertainty.

\textbf{Contributions.} We present S-Lemma Uncertainty Ellipsoid estimation (SLUE), a relaxation hierarchy to compute explicit and expressive pose uncertainty bounds. As the name suggests, SLUE relies on a generalization of the classical S-lemma~\cite{Polik07-SLemma} to
optimize directly for the shape of an ellipsoidal bound on pose uncertainty. SLUE is guaranteed to give an outer bound on the pose constraints implied by keypoint uncertainty. The first-order relaxation gives a fast but conservative uncertainty bound, while the hierarchy trades computation for tighter bounds. We reduce the computational cost by using infinity-norm keypoint constraints and a quaternion formulation. The SLUE bound describes pose uncertainty in an interpretable manner which is easy to project to independent translation and axis-angle orientation bounds. We demonstrate SLUE on two object-focused datasets~\cite{Brachmann14cvpr,Xiang17rss-posecnn} and a real-world drone tracking scenario~\cite{Shaikewitz24ral-CAST}.

To summarize, our contributions are:
\begin{itemize}
    \item An efficient and statistically rigorous algorithm for pose uncertainty quantification using an ellipsoidal outer bound.
    \item A relaxation hierarchy guaranteed to converge to the minimum volume ellipsoidal uncertainty bound at a given confidence.
    \item A projection scheme to reduce joint pose uncertainty bounds to interpretable translational and angular uncertainty bounds.
    \item Extensive experiments on three real-world datasets.
\end{itemize}

The rest of the paper is organized as follows. We discuss related work and introduce notation in Sections \ref{sec:related} and \ref{sec:notation}, respectively. We then describe our keypoint noise model and develop the uncertainty estimation problems in \prettyref{sec:problem}. We relax the minimum volume ellipsoid problem into a tractable uncertainty estimation problem using the generalized S-lemma in \prettyref{sec:uncertainty}, and generalize this approach to a hierarchy of relaxations guaranteed to converge to the minimum ellipsoid bound in \prettyref{sec:hierarchy}. In \prettyref{sec:uncertainty_bounds} we reduce ellipsoidal uncertainty into interpretable bounds on translation and orientation. Lastly, we evaluate our method on three real-world datasets in \prettyref{sec:experiments}.

 \section{Related Work}
\label{sec:related}

\textbf{Object Pose Uncertainty.}
There are several popular heuristics for pose uncertainty quantification. Ensemble methods~\cite{Shi21arxiv-FastUQPose, Wursthorn24-ensemblePoseUQ} aggregate multiple pose estimates into an approximation of estimator variance. Such a strategy has no grounding in the quality of the estimator and thus no statistical guarantees. A more precise approach is to discretize the object into particles to track~\cite{Deng19rss-poserbpf} or distinct regions~\cite{Hodan20cvpr-epos} and model uncertainty locally. This and other continuous correspondence ambiguity approaches~\cite{Okorn20iros-OrientationDistributions,Haugaard22cvpr-surfemb} are generally used to capture orientation uncertainty in the case of symmetric objects. 
However, they come with no uncertainty guarantees and do not easily translate into explicit uncertainty bounds.
The most common way uncertainty estimation is integrated into pose estimation pipelines is through learned estimates, which are often used to select the best pose in a multi-hypothesis setting~\cite{Liu24arxiv-poseSurvey, Brachmann16cpvr-UncertaintyPose,Tian20icra-robustPoseRGBD, wen24cvpr-foundationPose, Lin22icra-keypointTracking}. Among these works~\cite{Lin22icra-keypointTracking} is most similar to ours, using a neural network to propagate estimated keypoint uncertainty to a pose. All of these approaches, however, lack rigorous statistical guarantees.

We draw on an emerging body of statistically rigorous distribution-free pose uncertainty quantification from conformal prediction~\cite{Shafer08jmlr-TutorialConformal}. In particular, Yang and Pavone~\cite{Yang23arxiv-ransag} compute conformal bounds on keypoint error and use a semidefinite relaxation to maximize the pose error consistent with the keypoint bounds. Follow-up works improve uncertainty estimation with sample-based inner approximation~\cite{Yang24rss-closure} (thus sacrificing uncertainty guarantees) or simplify to a generalized Chebyshev center problem~\cite{Tang24l4dc-setMembership}. A key limitation of these approaches is expressiveness: \cite{Yang23arxiv-ransag, Tang24l4dc-setMembership} fix the shape of the outer bound (\ie consider spherical bounds) and optimize only for scale. This hides the effects of scale ambiguity and leads to conservative bounds. A simpler approach for fixed uncertainty bound shape is to apply conformal prediction directly to the pose estimate. In this paper, we solve a more expressive relaxation that captures scale ambiguities by optimizing directly for the shape of a statistically rigorous outer uncertainty bound.

\textbf{Minimal Bounding Sets.}
Reducing a compact set into a simple geometric object, such as a minimum volume bounding ellipsoid, has major computational and interpretability advantages~\cite{Lasserre15-ellipsoidBounding}. When the set is convex, finding a bounding ellipsoid of minimal volume is known as the Löwner-John ellipsoid problem~\cite{BenTal01-ConvexOpt} and admits a solution using convex optimization. The non-convex case is more difficult. Nie et al. \cite{Nie05optim-MinimalEnclosingEllipsoid} propose a sum-of-squares relaxation hierarchy which always returns an ellipsoidal outer bound and converges to the minimum-trace bounding ellipsoid. The log determinant objective, which gives minimum volume, is dismissed as impractical. Other works~\cite{Casini14tac-SetUncertainty,Tang24l4dc-setMembership} rely on strong simplifying assumptions. In particular,~\cite{Tang24l4dc-setMembership,Yang23arxiv-ransag} fix the shape of the ellipsoid bound, trading expressiveness for certifiably optimal scale estimation. In this paper, we give a new perspective on the sum-of-squares approach based on the generalized S-Lemma. Under our approach, the log determinant objective is numerically stable and we directly solve for the minimum volume bound.

 \section{Notation}
\label{sec:notation}
Symbols for vectors and matrices are bolded to distinguish them from scalars. When there are multiple subscripts, we wrap a vector or matrix in parentheses to indicate specific elements. For example, $(\epsilon_i)_1$ denotes the first element of the vector $\'\epsilon_i$. For integer indexing, we use the shorthand $[n]\triangleq {1,...,n}$. We write the $i$th 3D basis vector as $\*{\hat e}_i$ for $i\in[3]$. We denote the Kronecker product with $\otimes$ and use the $\mathrm{vec}(\cdot)$ operator to denote the vectorization of a matrix by stacking its columns. The set $\mathcal{S}^n$ is all symmetric $n\times n$ matrices and the set $\mathbb{S}^n$ is any $n$-dimensional unit vector.

\textbf{Polynomials.} For non-negative integer $\kappa$, we denote the $\kappa$-order monomial basis in variable $\*x\in\mathbb{R}^n$ by $[\*x]_\kappa\in\mathbb{R}^{C(n + \kappa - 1, \kappa)}$, where $C(\cdot, \cdot)$ is shorthand for the binomial coefficient. For example, $[\*x]_2$ is defined as:
\begin{equation}
    \label{eq:secondorderbasis}
    [\*x]_2 \triangleq \begin{bmatrix}
        x_1^2 & x_1 x_2 & \hdots & x_n^2
    \end{bmatrix}^\T\!.
\end{equation}
Note that $[\*x]_1 = \*x$ and $[\*x]_0 = 1$. We denote the set of polynomials in $\*x$ of order at most $\kappa$ by $\RR_{\kappa}[\*x]$. Let $x_1 = 1$ so that $[\*x]_\kappa$ contains all monomials up to order $\kappa$. Then, any $p\in\RR_{2\kappa}[\*x]$ may be written in quadratic form as $p(\*x) = [\*x]_\kappa^\T \*A [\*x]_\kappa$ for some (possibly non-unique) constant matrix $\*A$. We use the notation $p(\*x) \succeq_{sos} 0$ to mean $p(\*x)$ may be represented by some matrix $\*A \succeq 0$. %

\begin{figure}[t!]
    \centering
    \includegraphics[width=0.99\linewidth]{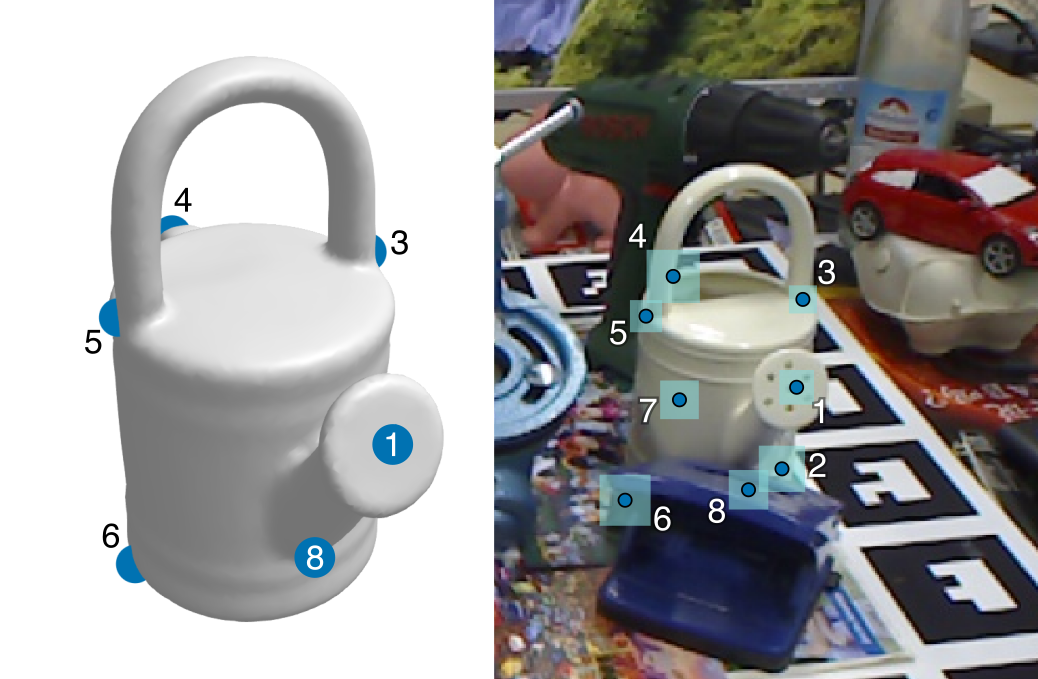}
\caption{\textbf{Keypoint measurements.} Given a 3D model with annotated 3D keypoints (left), we assume pixel detections of the location of each keypoint in the image frame (right). Pixel keypoint measurements also carry an uncertainty bound, shaded in blue.}
    \label{fig:cad}
\end{figure}

\section{Problem Formulation: Pose Uncertainty Estimation from Bounded Keypoint Noise}
\label{sec:problem}
We seek an estimate of object pose uncertainty from RGB pixel measurements of object keypoints and their corresponding keypoint uncertainty sets (see \prettyref{fig:cad}). 
We describe this measurement uncertainty model in \prettyref{sec:meas} and propagate keypoint uncertainty to pose uncertainty constraints in \prettyref{sec:poseuncertaintyset}. Distinct from~\cite{Yang23arxiv-ransag}, we consider infinity-norm keypoint uncertainty and explicitly propagate the coverage probability to the pose uncertainty constraint set.

\begin{figure*}[t!]
    \centering
\begin{tikzpicture}[scale=1.0]

\begin{scope}[shift={(-5,0)}]  

\draw[thick, black] (-2,-2) rectangle (2,2);
            \node[above, font=\small] at (0,2.2) {Image Plane};

\draw[gray, thick, dotted] (-0.8, 0.9) -- (1.0, 0.1);
            \draw[gray, thick, dotted] (-0.8, 0.9) -- (0, -1.3);
            \draw[gray, thick, dotted] (-0.8, 0.9) -- (0.7, 0.7);
            \draw[gray, thick, dotted] (0, -1.3) -- (1.0, 0.1);
            \draw[gray, thick, dotted] (0.7, 0.7) -- (1.0, 0.1);
            \draw[gray, thick, dotted] (0, -1.3) -- (0.7, 0.7);
            
\fill[red] (-0.8, 0.9) circle (2pt);
            \draw[red, thick] (-1.1, 0.6) rectangle (-0.5, 1.2);
            \node[red, font=\tiny] at (-0.82, 0.44) {$\*y_1$};
            
\fill[blue] (1.0, 0.1) circle (2pt);
            \draw[blue, thick] (0.85, -0.05) rectangle (1.15, 0.25);
            \node[blue, font=\tiny] at (1.05, -0.2) {$\*y_2$};
            
\fill[green!70!black] (0, -1.3) circle (2pt);
            \draw[green!70!black, thick] (-0.25, -1.55) rectangle (0.25, -1.05);
            \node[green!70!black, font=\tiny] at (0.05, -1.7) {$\*y_3$};
            
\fill[purple] (0.7, 0.7) circle (2pt);
            \draw[purple, thick] (0.5, 0.5) rectangle (0.9, 0.9);
            \node[purple, font=\tiny] at (0.7, 1.04) {$\*y_4$};
            
\node[font=\small] at (2.3, 0) {$v$};
            \node[font=\small] at (0, -2.3) {$u$};
        \end{scope}

\draw[->, ultra thick] (0.5,0) -- (1.5,0);
        \node[above, font=\small] at (1,0.2) {Backprojected 3D points};
        \node[below, font=\small] at (1,-0.2) {for each 2D keypoint};

\begin{scope}[shift={(7,-2)}, scale=0.7]  \tdplotsetmaincoords{60}{120}
        \begin{scope}[tdplot_main_coords]
\draw[->] (0,0,0) -- (2,0,0) node[anchor=east, font=\small]{$X$};
            \draw[->] (0,0,0) -- (0,2,0) node[anchor=west, font=\small]{$Y$};
            \draw[->] (0,0,0) -- (0,0,8) node[anchor=south, font=\small]{$Z$};

\coordinate (O) at (0,0,0);
            \node[below right, font=\small] at (0,0.3,0.7) {Camera};
            \fill[black] (O) circle (2pt);

            \coordinate (Y1_1) at (-1.2,1.2,6);
            \coordinate (Y2_1) at (1.84,0.08,6);
            \coordinate (Y3_1) at (0.16,-4.06,4.8);
            \coordinate (Y4_1) at (1.96,1.78,5.6);

            \fill[black] (Y1_1) circle (2pt);
            \fill[black] (Y2_1) circle (2pt);
            \fill[black] (Y3_1) circle (2pt);
            \fill[black] (Y4_1) circle (2pt);

            \draw[black!80, thick] (Y4_1) -- (Y1_1);
            \draw[black!80, thick] (Y4_1) -- (Y2_1);
            \draw[black!80, thick] (Y3_1) -- (Y2_1);
            \draw[black!80, thick] (Y3_1) -- (Y1_1);
            \draw[black!80, thick] (Y2_1) -- (Y1_1);
            \draw[black!80, thick] (Y3_1) -- (Y4_1);

\coordinate (A2_1) at (-0.9, 2.1, 7.5);
            \coordinate (B2_1) at (-2.1, 2.1, 7.5);
            \coordinate (C2_1) at (-2.1, 0.9, 7.5);
            \coordinate (D2_1) at (-0.9, 0.9, 7.5);

\fill[white, opacity=0.7] (A2_1) -- (B2_1) -- (C2_1) -- (D2_1);
            
\fill[red!30, opacity=0.7] (O) -- (A2_1) -- (B2_1);
            \fill[red!40, opacity=0.7] (O) -- (D2_1) -- (A2_1);
            
\draw[red!80, thick] (C2_1) -- (D2_1);
            \draw[red!80, thick] (C2_1) -- (B2_1);
            
\draw[red!80, thick] (O) -- (A2_1);
            \draw[red!80, thick] (O) -- (B2_1);
            \draw[red!80, thick] (O) -- (D2_1);
            \draw[red!80, thick] (A2_1) -- (B2_1);
            \draw[red!80, thick] (A2_1) -- (D2_1);
\begin{pgfonlayer}{bg}
                \draw[red!80, thick, dotted] (O) -- (C2_1);
            \end{pgfonlayer}
            
\coordinate (A2_2) at (3.0, 0.5, 7.5);
            \coordinate (B2_2) at (2.0, 0.5, 7.5);
            \coordinate (C2_2) at (2.0, -0.5, 7.5);
            \coordinate (D2_2) at (3.0, -0.5, 7.5);

\fill[white, opacity=0.7] (A2_2) -- (B2_2) -- (C2_2) -- (D2_2);
            
\fill[blue!30, opacity=0.7] (O) -- (A2_2) -- (D2_2) -- cycle;
            \fill[blue!40, opacity=0.7] (O) -- (B2_2) -- (A2_2) -- cycle;
            
\draw[blue!80, thick] (C2_2) -- (B2_2);
            \draw[blue!80, thick] (C2_2) -- (D2_2);
            
\draw[blue!80, thick] (O) -- (A2_2);
            \draw[blue!80, thick] (O) -- (B2_2);
            \draw[blue!80, thick] (O) -- (D2_2);
            \draw[blue!80, thick] (A2_2) -- (B2_2);
            \draw[blue!80, thick] (A2_2) -- (D2_2);
\begin{pgfonlayer}{bg}
                \draw[blue!80, thick, dotted] (O) -- (C2_2);
            \end{pgfonlayer}
            
\coordinate (A2_3) at (0.875, -4.2, 6);
            \coordinate (B2_3) at (-0.875, -4.2, 6);
            \coordinate (C2_3) at (-0.875, -5.95, 6);
            \coordinate (D2_3) at (0.875, -5.95, 6);

\fill[white, opacity=0.7] (A2_3) -- (B2_3) -- (C2_3) -- (D2_3);
            
\fill[green!30, opacity=0.7] (O) -- (D2_3) -- (A2_3);
            \fill[green!40, opacity=0.7] (O) -- (A2_3) -- (B2_3);
            
\draw[green!70!black, thick] (O) -- (A2_3);
            \draw[green!70!black, thick] (O) -- (D2_3);
            \draw[green!70!black, thick] (O) -- (B2_3);
            \draw[green!70!black, thick] (A2_3) -- (B2_3);
            \draw[green!70!black, thick] (B2_3) -- (C2_3);
            \draw[green!70!black, thick] (A2_3) -- (D2_3);
            \draw[green!70!black, thick] (C2_3) -- (D2_3);
\begin{pgfonlayer}{bg}
                \draw[green!70!black, thick, dotted] (O) -- (C2_3);
            \end{pgfonlayer}
            
\coordinate (A2_4) at (2.86, 2.86, 7);
            \coordinate (B2_4) at (1.59, 2.86, 7);
            \coordinate (C2_4) at (1.59, 1.59, 7);
            \coordinate (D2_4) at (2.86, 1.59, 7);

\fill[white, opacity=0.7] (A2_4) -- (B2_4) -- (C2_4) -- (D2_4);
            
\fill[purple!40, opacity=0.7] (O) -- (A2_4) -- (B2_4);
            \fill[purple!35, opacity=0.9] (O) -- (A2_4) -- (D2_4);
            
\draw[purple!80, thick] (C2_4) -- (B2_4);
            \draw[purple!80, thick] (C2_4) -- (D2_4);
            
\draw[purple!80, thick] (O) -- (A2_4);
            \draw[purple!80, thick] (O) -- (B2_4);
            \draw[purple!80, thick] (O) -- (D2_4);
            \draw[purple!80, thick] (A2_4) -- (B2_4);
            \draw[purple!80, thick] (A2_4) -- (D2_4);
\begin{pgfonlayer}{bg}
                \draw[purple!80, thick, dotted] (O) -- (C2_4);
            \end{pgfonlayer}

        \end{scope}
        \end{scope}

    \end{tikzpicture}
\caption{\textbf{Pose Uncertainty Constraint Set.} The infinity-norm bounds on keypoint error (left) each imply a cone of backprojected 3D feasible keypoint positions (right). Combining the bounds for multiple keypoints and imposing object shape constraints yields an implicit \emph{pose uncertainty constraint set} which contains many feasible poses.}
    \label{fig:conformal_calibration}
\end{figure*}
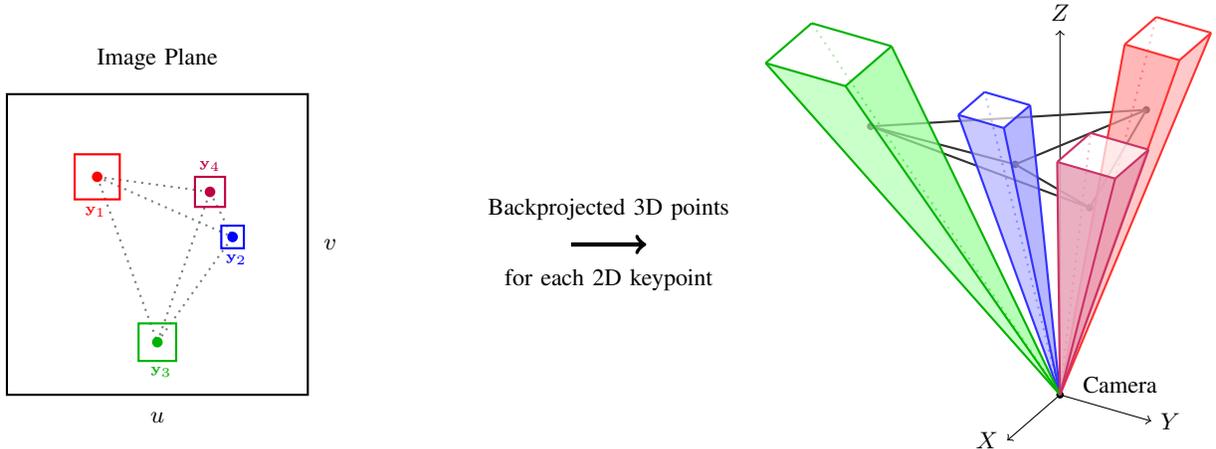

\subsection{Measurement Model: Keypoint Uncertainty Sets}
\label{sec:meas}
Given an RGB image of a target object, we measure the pixel locations of a sparse set of 3D CAD-model-frame \emph{keypoints} (\prettyref{fig:cad}). For each 3D keypoint, we use a front-end (\ie a neural network) to detect its noisy pixel location in the image frame~\cite{He17iccv-maskRCNN}. We model measurement noise using \emph{keypoint uncertainty sets}. That is, we assume the ground truth pixel keypoint is within a norm-ball about the detected pixel keypoint with high probability. In practice, we obtain these sets using conformal prediction (\prettyref{sec:keypoints}). This gives the following bounded uncertainty model.

Denote the fixed 3D model-frame keypoints by $\*b_i\in\RR^3$ for $i\in[N]$ and their corresponding pixel measurements by $\*y_i = [u_i, v_i, 1]^\T$. For an object with position $\*t\in\RR^3$ and orientation $\*R\in\SO3$ in the frame of a camera with intrinsics $\*K$, the keypoint measurement model reprojects the 3D points and adds arbitrary noise:
\begin{equation}
    \label{eq:conformal_meas}
    \*y_i = \frac{\*K(\*R\*b_i + \*t)}{\*{\hat{e}}_3\cdot(\*R\*b_i + \*t)} + \'\eps_i,
\end{equation}
where $\'\eps_i\in\RR^2$ is some measurement noise in homogeneous form ($(\eps_i)_3\equiv0$).

We make no distributional assumptions for $\'\eps_i$. Instead, assume its $p$-norm is bounded with probability $1-\alpha_i$:
\begin{equation}
    \label{eq:conformal_boundedprob}
    \mathbb{P}(\|\'\eps_i\|_p \leq r_i(\alpha_i)) \geq 1 - \alpha_i.
\end{equation}
In practice, we obtain the radii $r_i(\alpha_i)$ for a given confidence $\alpha_i$ using split conformal prediction~\cite{Angelopoulos24-ConformalTextbook} (detailed in \prettyref{sec:keypoints}). In this paper we treat $r_i(\alpha_i)$ as a given constant and drop the $\alpha_i$ dependence. In the following, we specialize to the case $p=\infty$, which corresponds to square axis-aligned uncertainty sets with side length $2r_i$ (see~\prettyref{fig:conformal_calibration}). As we will show, the $\infty$-norm has strong computational advantages over the $2$-norm used in~\cite{Yang23arxiv-ransag}; we give the $p=2$ case in~\prettyref{appendix:2norm} for completeness.

\subsection{Uncertainty Constraint Set for 6D Pose}
\label{sec:poseuncertaintyset}
Given the measurement model~\eqref{eq:conformal_meas} and noise bound~\eqref{eq:conformal_boundedprob} we now derive a \emph{pose uncertainty constraint set} that contains the true pose with high probability. Combining eqs.~(\ref{eq:conformal_meas},~\ref{eq:conformal_boundedprob}), the following \emph{reprojection} constraint holds with probability at least $1-\alpha_i$:
\begin{equation}
    \label{eq:yminusgt}
    \left\|\*y_i - \frac{\*K(\*R\*b_i + \*t)}{\*{\hat{e}}_3\cdot(\*R\*b_i + \*t)}\right\|_\infty = \|\'\eps_i\|_\infty \leq r_i.
\end{equation}

It is also reasonable to assume detected keypoints are visible to the camera. Thus, the projection of $\*b_i$ must have positive depth. This is the \emph{chirality} (front-of-camera) constraint for keypoint $i$:
\begin{equation}
    \label{eq:frontofcamera}
    \tag{FoC}
    \*{\hat{e}}_3\cdot(\*R\*b_i + \*t) > 0.
\end{equation}

To proceed we convert the rational reprojection constraint~\eqref{eq:yminusgt} to a polynomial \emph{backprojection} constraint. The infinity norm constrains each coordinate independently. Multiplying~\eqref{eq:yminusgt} by the depth and using chirality~\eqref{eq:frontofcamera} to drop its absolute value, eq.~\eqref{eq:yminusgt} reduces to two inequality constraints which hold with probability at least $1-\alpha_i$:
\begin{equation}
    \label{eq:bpinf}
    \tag{$\text{BP}_\infty$}
    \begin{aligned}
        \left|
        (\*y_i\*{\hat{e}}_3^\T - \*K)(\*R\*b_i + \*t) 
        \cdot\*{\hat{e}}_j \right|
        \leq r_i\*{\hat{e}}_3^\T(\*R\*b_i + \*t)\\
    \end{aligned},\!
\end{equation}
for $j=1,2$. Each absolute value yields two constraints linear in $\*R$ and $\*t$; thus, \eqref{eq:bpinf} is four linear inequalities.

\textbf{Pose Uncertainty Constraint Set.}
Combining the backprojection and chirality constraints for each keypoint gives the following \emph{pose uncertainty constraint set}, visualized in \prettyref{fig:conformal_calibration}.

\begin{proposition}[Pose Uncertainty Constraint Set]
    \label{prop:pus}
    Assume measurements of $N$ object keypoints of the form~\eqref{eq:conformal_meas} and noise bounded in infinity-norm with high probability as~\eqref{eq:conformal_boundedprob}. The true position $\*t_\mathrm{gt}$ and orientation $\*R_\mathrm{gt}$ of the object are contained in the following constraint set:
    \begin{equation}
        \label{eq:purse}
        \tag{$\mathcal{P}_\infty$}
        \left\{
            \begin{array}{ll}
                \*t\in\RR^3\\
                \*R\in\SO3
            \end{array}
            \left|\vphantom{\sum}\right.
            \begin{array}{ll}
                \eqref{eq:frontofcamera}_i\\
                \eqref{eq:bpinf}_i
            \end{array}
            \text{ for }
            i \in [N]
        \right\}\!,
    \end{equation}
    with probability at least $\beta$.

    For arbitrary dependence among $\'\eps_i$, $\beta \geq 1-\sum_{i=1}^N\alpha_i$. In the extreme where $\'\eps_i$ is independent from $\'\eps_j$ for all $i\neq j$ then $\beta=\prod_{i=1}^N(1-\alpha_i)$. When the $\'\eps_i$ are perfectly positively correlated and all $\alpha_i \equiv \alpha$, we have $\beta=1-\alpha$.

    \begin{proof}
        For perfect correlation the probability of~$\eqref{eq:bpinf}_i$ for all $i$ is the same as the probability for any choice of $i$. Under independence, the probability of the intersection of~$\eqref{eq:bpinf}_i$ for each $i$ reduces to a product. For arbitrary dependence, the result follows from a union bound.
\end{proof}
\end{proposition}

Contrary to~\cite{Yang23arxiv-ransag}, we observe the statistical guarantee of covering the ground truth pose holds with probability less than $1-\alpha$ in general. The worst case confidence can quickly become uninformative. For example, if $N=10$ and $\alpha_i\equiv\alpha=0.1$ for all $i$, the set~\eqref{eq:purse} is guaranteed to contain the ground truth with probability greater than $0$. In practice, however, we expect some positive correlation; thus, coverage should generally exceed the independence bound. For our choice of $\alpha$ and $N$, the independence bound gives a $35\%$ coverage probability. We provide empirical coverage results in~\prettyref{sec:experiments}.

\begin{remark}[Low-Confidence Keypoints and Outliers]
    Low-confidence or missing keypoints have large or infinite noise bounds, respectively, and thus have minimal effect on \eqref{eq:purse}. Outlier keypoints only skew \eqref{eq:purse} if they occur with low probability ($\leq \alpha_i$), which is accounted for in the coverage probability of \prettyref{prop:pus}.
\end{remark}

The set~\eqref{eq:purse} is an \emph{implicit} representation of the set of poses which are consistent with the backprojection and chirality constraints (see \prettyref{fig:conformal_calibration}). We first seek a simpler characterization of this set, choosing to model uncertainty as an ellipsoidal bound centered on a single pose estimate (as in \prettyref{fig:bounding_ellipsoid}). Our procedure is agnostic to the choice of pose estimator, and the estimate itself does not need to be inside \eqref{eq:purse}. We use standard perspective-n-point~\cite{Terzakis20eccv-sqpnp} to generate a pose estimate $(\*{\bar R}, \*{\bar t})$ (see~\prettyref{appendix:conformal_exp}). Given an estimated pose, we seek a single ellipsoid which bounds the pose uncertainty.

\begin{problem}[Joint Ellipsoidal Bound]
    \label{prob:joint}
    Reduce the pose constraints~\eqref{eq:purse} into a single ellipsoidal uncertainty bound of minimal volume (\prettyref{fig:bounding_ellipsoid}) which is centered at the pose estimate and contains the ground truth pose with probability at least $\beta$. That is, find $\*H\succ0$ such that $(\*R, \*t)\in\eqref{eq:purse}$ implies:
    \begin{equation}
        \label{eq:ellipse_understandable}
        \begin{bmatrix}
            \mathrm{vec}(\*R - \*{\bar R})\\
            \*t - \*{\bar t}
        \end{bmatrix}^\T
        \*H
        \begin{bmatrix}
            \mathrm{vec}(\*R - \*{\bar R})\\
            \*t - \*{\bar t}
        \end{bmatrix}
        \leq 1.
    \end{equation}
\end{problem}

\prettyref{prob:joint} seeks a simple convex representation of the pose uncertainty constraints~\eqref{eq:purse}. We also desire explicit bounds on translation and axis-angle uncertainty.

\begin{problem}[Rotation and Translation Bounds]
    \label{prob:uncertainty}
    Reduce the joint ellipsoidal bound~\eqref{eq:ellipse_understandable} into translational and angular ellipsoids centered at the pose estimate $(\*{\bar R}, \*{\bar t})$ and containing the ground truth pose with probability at least $\beta$.
\end{problem}

 \section{Uncertainty Ellipsoid via S-Lemma}
\label{sec:uncertainty}
In this section we compute the ellipsoidal uncertainty bound~\eqref{eq:ellipse_understandable} from~\prettyref{prob:joint} (see \prettyref{fig:bounding_ellipsoid}a). Our procedure, motivated by a generalized S-lemma, offers a new and more intuitive perspective on the sum-of-squares approach for a minimum volume bounding ellipsoid~\cite{Nie05optim-MinimalEnclosingEllipsoid}. We first write the pose constraints in quadratic form (\prettyref{sec:quadform}) and develop an S-lemma relaxation for the bounding ellipsoid problem in \prettyref{sec:boundingellipsoid}. In \prettyref{sec:hierarchy} we extend this to a hierarchy of relaxations which is guaranteed to converge to the true minimum volume ellipsoid. Our approach is faster and more expressive than the outer approximations in~\cite{Tang24l4dc-setMembership,Yang23arxiv-ransag}, optimizing directly for ellipsoid shape. We name the combined method \emph{SLUE}, for S-Lemma Uncertainty Ellipsoid estimation. SLUE is summarized in \prettyref{alg:slue}.

\begin{algorithm}[t]
    \SetKwComment{tcp}{$\triangleright$ }{}
    \SetKwComment{tcc}{// }{}
    \DontPrintSemicolon

    \KwIn{Keypoint measurements $\*y_i$, keypoint bounds $r_i$ satisfying \eqref{eq:conformal_boundedprob}, $i\in[N]$, center $(\*{\bar R}, \*{\bar t})$, relaxation order $\kappa$.}
    \KwOut{Pose uncertainty ellipsoids $\*H$, $\*H_t$, $\*H_\theta$.}
\tcc{form pose constraint set \eqref{eq:purse}}
    \For{$i\leftarrow 1$ \KwTo $N$}{
        $\*A_i^{(1)}\leftarrow$ {\small \textsc{chirality}}$(\*y_i, r_i)$ 
        \tcp*[r]{\textnormal{Eq. \eqref{eq:frontofcamera}}}
        $\*A_i^{(2:5)} \leftarrow$ {\small \textsc{backproj}}$(\*y_i, r_i)$
        \tcp*[r]{\textnormal{Eq. \eqref{eq:bpinf}}}
    }
    \tcc{uncertainty ellipsoids}
    $\*H \leftarrow$ {\small \textsc{solveSDP}}$(\*A, \*{\bar R}, \*{\bar t}, \kappa)$
    \tcp*[r]{\textnormal{\prettyref{prop:bounding_hierarchy}}}
    $\*H_t \leftarrow$ $\left(\*P_t \*H^{-1} \*P_t\right)^{-1}$
    \tcp*[r]{\textnormal{Eq. \eqref{eq:marginalized_translations}}}
    $\*H_\theta \leftarrow$ {\small \textsc{angProject}}$(\*H, \*{\bar R})$
    \tcp*[r]{\textnormal{\prettyref{prop:ang_bounds}}}

    \caption{S-Lemma Uncertainty Ellipsoids}
    \label{alg:slue}
\end{algorithm}

\subsection{Pose Uncertainty Set in Quadratic Form}
\label{sec:quadform}
First, we write~\eqref{eq:purse} as a set of quadratic constraints in the homogenized variable $\*x \triangleq [1, \*r^\T, \*t^\T]^\T$, where $\*r \triangleq \mathrm{vec}(\*R)\in\RR^9$. To write the chirality and backprojection constraints in terms of $\*r$ we use the Kronecker identity: $\*M\*R\*b_i = (\*b_i^\T \otimes \*M)\*r$ for $\*M=\*K$ or $\*M=\eye_3$. This gives the following quadratic form of~\eqref{eq:purse}:
\begin{equation}
    \label{eq:purse_qcqp}
    \left\{
        \begin{array}{cc}
            \*x\in\RR^{13}\\
            x_{1} = 1
        \end{array}
        \left|\vphantom{\sum_{i}^N}\right.
        \begin{array}{rll}
            \*x^\T\*A_i\*x\hspace{-6pt} &\leq 0,& i\in[5N]\\
            \*x^\T\*Q_j\*x\hspace{-6pt} &= 0,& j\in[15]
        \end{array}
    \right\}\!.
\end{equation}
In~\eqref{eq:purse_qcqp}, the equality constraints enforce $\*R\in\SO{3}$ and the inequalities enforce the backprojection and chirality constraints. We give explicit $\*A_i$ and $\*Q_j$ in~\prettyref{appendix:roteq}.

\subsection{Bounding Ellipsoid via S-Lemma}
\label{sec:boundingellipsoid}
\begin{figure*}[tb]
    \centering
    \subfloat[$\kappa=1$]{\includegraphics[width=0.22\linewidth]{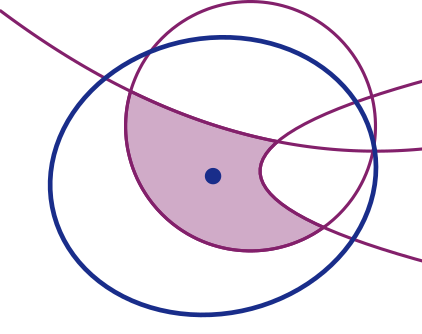}}
    \hspace{0.2cm}
    \subfloat[$\kappa=2$]{\includegraphics[width=0.22\linewidth]{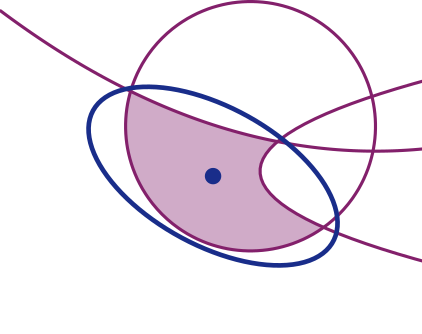}}
    \hspace{0.2cm}
    \subfloat[$\kappa=3$]{\includegraphics[width=0.22\linewidth]{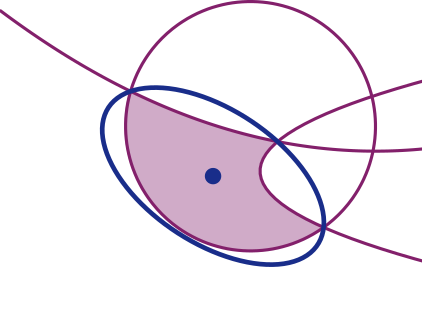}}
    \hspace{0.2cm}
    \subfloat[$\kappa=4$]{\includegraphics[width=0.22\linewidth]{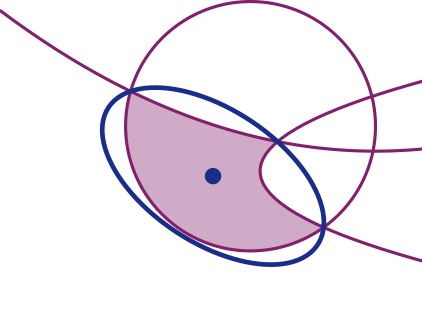}}
    \caption{\textbf{Hierarchy of Bounding Ellipsoids.} We solve for an ellipsoidal representation \eqref{eq:ellipse_understandable}, in blue, of a set defined by several quadratic constraints \eqref{eq:purse}, in purple, which may be non-convex. Our approach (SLUE) admits a hierarchy of ellipsoidal bounds guaranteed to converge to the minimum volume bound as relaxation order $\kappa$ tends to infinity. In this 2D toy example, the relaxation converges by $\kappa=3$.}
    \label{fig:bounding_ellipsoid}
\end{figure*}

Recalling~\prettyref{prob:joint}, we seek the minimum volume ellipsoid which outer-bounds the set~\eqref{eq:purse_qcqp} centered at the pose estimate ($\*{\bar R}$, $\*{\bar t}$). Let $\*{\bar r}\triangleq \mathrm{vec}(\*{\bar R})$ and $\*{\bar x} \triangleq [\*{\bar r}^\T, \*{\bar t}^\T]^\T$. In quadratic form, we seek $\*H$ such that:

\begin{equation}
    \label{eq:ellipse_problem}
    \eqref{eq:purse_qcqp}
    \implies
    \*x^\T
    \underbrace{
    \begin{bmatrix}
        \*{\bar x}^\T\*H\*{\bar x} - 1 & -\*{\bar x}^\T \*H \\
        -\*H\*{\bar x} & \*H
    \end{bmatrix}
    }_{\triangleq \*W(\*H)}
    \*x
    \leq 0.
\end{equation}

In words, any pose that satisfies \eqref{eq:ellipse_problem} and hence is in \eqref{eq:purse} must also be in the ellipsoid \eqref{eq:ellipse_understandable}. Recall that $\*H\succ 0$ is an unknown matrix defining the size and shape of the ellipsoid. To find the minimum volume $\*H$ satisfying~\eqref{eq:ellipse_problem}, we turn to a generalization of the celebrated S-lemma~\cite{Polik07-SLemma}.

\begin{proposition}[Generalized S-Lemma]
    \label{prop:slemma}
    Let $\*x\in\RR^n$ and $\*W,\,\*Y_i,\,\*Z_j\in\mathcal{S}^{n}$ be given matrices for $i\in[N]$ and $j\in[M]$. For the following statements, $\text{(ii)} \implies \text{(i)}$.
    \begin{enumerate}[(i)]
        \item \emph{(primal)} $\*x^\T\*Y_i\*x\leq 0$ and $\*x^\T\*Z_j\*x = 0$ for $i\in[N]$ and $j\in[M]$
$\implies \*x^\T\*W\*x\leq 0$.
        \item \emph{(dual)} There exist $\'\mu\in\RR^M$ and $\'\lambda\in\RR^N$, $\lambda_i \geq 0$ $\forall\ i$ such that $\*W \preceq \sum_{i=1}^N \lambda_i \*Y_i + \sum_{j=1}^M \mu_j \*Z_j$.
    \end{enumerate}
    \begin{proof}
        Pre and post-multiplying (ii) by $\*x^\T$ and $\*x$ with non-negative $\'\lambda$, statement (i) holds by substitution.
\end{proof}
\end{proposition}

The classical S-lemma states that the primal and dual are equivalent ($\text{(ii)} \iff \text{(i)}$) when there is only one inequality constraint: $N=1$ and $M=0$. \prettyref{prop:slemma} (also called the \emph{lossy} S-lemma~\cite{Lessard21-SLemma}) extends this to multiple constraints at the cost of a possible duality gap.

To apply~\prettyref{prop:slemma} to our problem, notice that the primal \emph{(i)} corresponds exactly to~\eqref{eq:ellipse_problem}, while the dual \emph{(ii)} provides a semidefinite constraint which may be incorporated into a convex optimization problem. Thus, we relax the constraint~\eqref{eq:purse_qcqp} to its convex dual \emph{(ii)} and optimize for the minimum volume ellipsoid satisfying \emph{(ii)}. The ellipsoid's volume is inversely proportional to $\log\det(\*H)$~\cite{Boyd04book}, leading to the following result.

\begin{proposition}[Bounding Ellipsoid]
    \label{prop:bounding_ellipse}
    An outer bounding ellipsoid for the set~\eqref{eq:purse_qcqp} is given by the solution to the following convex optimization problem.
    \begin{maxi}
{\substack{
            \*H\in\mathcal{S}^{12},\,\*H\succeq 0\\
            \'\lambda\in\RR^{5N}, \,\'\mu\in\RR^{15}
        }}
{\hspace{-5pt}\log\det(\*H)}
        {\label{eq:conformal_ellipsoid}}{}
\addConstraint{
            \hspace{-12pt}
\*W(\*H)
        }{\preceq \sum_{i=1}^{5N} \lambda_i \*A_i + \sum_{j=1}^{15} \mu_j \*Q_j}
        \addConstraint{
            \lambda_i \geq 0
        }{,\quad i\in[5N].}
    \end{maxi}
    The maximizer $\*H^\star$ is an ellipsoid centered at $\*{\bar{x}}$ as~\eqref{eq:ellipse_understandable}.
\end{proposition}

The proof follows directly from~\prettyref{prop:slemma}. In applying \prettyref{prop:slemma}, we implicitly dropped the homogenization constraint $x_1=1$. Clearly, if statement $(ii)$ holds for general $\*x$ then it holds when $x_1=1$. The S-lemma relaxation, however, means we only guarantee \eqref{eq:conformal_ellipsoid} returns \emph{some} outer ellipsoidal bound (not necessarily the smallest) on the pose uncertainty constraint set. To improve this bound, the next section extends~\eqref{eq:conformal_ellipsoid} into a relaxation hierarchy, depicted in \prettyref{fig:bounding_ellipsoid}, which is guaranteed to converge to the true minimum volume ellipsoidal bound.

\subsection{The S-Lemma Relaxation Hierarchy}
\label{sec:hierarchy}
In the previous section we used the generalized S-lemma to compute an ellipsoid which bounds the pose uncertainty set \eqref{eq:purse}. Here, we use ideas from sum-of-squares (SOS) programming~\cite{Lasserre01siopt-LasserreHierarchy} to derive an SOS S-lemma and extend the ellipsoid procedure into a hierarchy of relaxations. This hierarchy, visualized in \prettyref{fig:bounding_ellipsoid}, enjoys a guarantee of convergence to the true minimum volume ellipsoid bounding \eqref{eq:purse}. Our algorithm is similar to~\cite{Nie05optim-MinimalEnclosingEllipsoid}; our formulation is based on the S-lemma and solves directly for the minimum ellipsoid volume. Further, our approach overcomes previously reported numerical difficulties~\cite{Tang24l4dc-setMembership}, solving to optimality with off-the-shelf solvers such as MOSEK~\cite{mosek}.

The key idea of the SOS S-lemma hierarchy is to use \emph{dual polynomials} in the S-lemma relaxation. We introduce this by analogy to \prettyref{prop:slemma}. Written as an SOS inequality, the linear matrix inequality in statement (ii) of \prettyref{prop:slemma} reads:
\begin{equation}
    \label{eq:lmisos}
    \*x^\T \*W \*x \preceq_{sos} \sum_{i=1}^N \lambda_i \*x^\T \*Y_i \*x + \sum_{j=1}^M \mu_j \*x^\T \*Z_j \*x.
\end{equation}

The SOS S-lemma simply generalizes the dual variables in \eqref{eq:lmisos} to polynomials $\lambda_i(\*x)$ and $\mu_j(\*x)$.

\begin{proposition}[SOS Generalized S-Lemma]
    \label{prop:slemma_sos}
    Let $\*x\in\RR^n$ and $\*W,\,\*Y_i,\,\*Z_j\in\mathcal{S}^{n}$ be given matrices for $i\in[N]$ and $j\in[M]$. Fix integer $\kappa \geq 0$.
    For the following statements, $\text{(ii)} \implies \text{(i)}$.
    \begin{enumerate}[(i)]
        \item \emph{(primal)} $\*x^\T\*Y_i\*x\leq 0$ and $\*x^\T\*Z_j\*x = 0$ for $i\in[N]$ and $j\in[M]$
        $\implies \*x^\T\*W\*x\leq 0$.
        \item \emph{(dual)} There exist polynomials $\lambda_i,\,\mu_j\in\RR_{2\kappa}[\*x]$ with $\lambda_i(\*x) \succeq_{sos} 0$ for $i\in[N]$ and $j\in[M]$ such that:
    \end{enumerate}
    \begin{equation}
        \label{eq:sosslem}
        \quad\quad
        \*x^\T \*W \*x \preceq_{sos} \sum_{i=1}^N \lambda_i(\*x) \*x^\T \*Y_i \*x + \sum_{j=1}^M \mu_j(\*x) \*x^\T \*Z_j \*x.
    \end{equation}

    The quantity $\kappa+1$ is called the \emph{relaxation order}.
    \begin{proof}
        For polynomial $p(\*x)\in\RR[\*x]$, $p(\*x) \preceq_{sos} 0 \implies p(\*x) \leq 0$. Thus, statement \emph{(i)} holds by substitution.
    \end{proof}
\end{proposition}

Similar to \prettyref{prop:slemma}, the SOS generalized S-lemma provides a tractable relaxation of the polynomial constraint set \eqref{eq:purse_qcqp}. Although the sum-of-squares notation makes it more difficult to interpret, \eqref{eq:sosslem} is simply a linear matrix inequality among matrices in $\mathcal{S}^d$, where $d$ is the dimension of the monomial basis $[\*x]_\kappa$. \prettyref{appendix:sosslem} gives an explicit example. The case $\kappa=0$ corresponds exactly to \eqref{eq:lmisos}, while larger $\kappa$ give higher-order relaxations. Notice that each $\kappa$ relaxation is a special case of the $\kappa+1$ relaxation with the highest order term of the dual polynomials set to zero. Thus, increasing $\kappa$ gives a more expressive relaxation. This provides an elegant way to trade off computation (the SOS inequality is a linear matrix inequality of size $d\times d$) for relaxation accuracy.

As in the first-order case, we use \prettyref{prop:slemma_sos} to derive a hierarchy of bounding ellipsoids.

\begin{proposition}[Bounding Ellipsoid Hierarchy]
    \label{prop:bounding_hierarchy}
    For fixed relaxation order $\kappa+1$, an outer bounding ellipsoid for the set~\eqref{eq:purse_qcqp} is given by the solution to the following convex optimization problem.
    \begin{maxi}
{\substack{
            \*H\in\mathcal{S}^{12},\,\*H\succeq 0\\
            \lambda_i,\,\mu_j\in\RR_{2\kappa}[\*x],\\
            i\in[5N],\,j\in[15]
        }}
{\hspace{-5pt}\log\det(\*H)}
        {\label{eq:conformal_ellipsoid_higherorder}}{}
\addConstraint{
            \hspace{-12pt}
\*x^\T\*W(\*H)\*x
        }{\preceq_{sos} \sum_{i=1}^{5N} \lambda_i(\*x) \*x^\T\*A_i\*x \\&&&\quad\ + \sum_{j=1}^{15} \mu_j(\*x) \*x^\T\*Q_j\*x}
        \addConstraint{
            \lambda_i \succeq_{sos} 0
        }{,\quad i\in[5N].}
    \end{maxi}
    The maximizer $\*H^\star_{\kappa}$ defines an ellipsoid in terms of $\*x$ centered at $\*{\bar{x}}$ as in~\eqref{eq:ellipse_understandable}. Further, for any $\kappa$ the $\kappa+1$ relaxation is an upper bound: $\log\det(\*H_\kappa^\star) \leq \log\det(\*H_{\kappa+1}^\star)$.
\end{proposition}

\prettyref{prop:bounding_hierarchy} simply applies the SOS generalized S-lemma (\prettyref{prop:slemma_sos}) to relax the bounding ellipsoid problem. It can be efficiently expressed as a semidefinite program by converting the SOS inequality into a linear matrix inequality. We can also show the relaxation is asymptotically \emph{tight} with $\kappa$.

\begin{theorem}[Hierarchy Convergence~\cite{Nie05optim-MinimalEnclosingEllipsoid,Tang24l4dc-setMembership}]
    \label{thm:convergence}
    The maximizer $\*H^\star_\kappa$ of \eqref{eq:conformal_ellipsoid_higherorder} gives the minimum volume ellipsoid bounding~\eqref{eq:purse_qcqp} as relaxation order $\kappa$ tends to infinity.
\end{theorem}

We prove \prettyref{thm:convergence} in \prettyref{appendix:sosslem}, drawing from~\cite{Nie05optim-MinimalEnclosingEllipsoid}. Although the theorem only guarantees asymptotic convergence, we observe finite convergence at low relaxation orders occurs in practice (see \eg \prettyref{fig:bounding_ellipsoid}) and finite convergence guarantees have been obtained in related problems~\cite{Nie14mp-finiteConvergenceLassere}. In this paper we only solve~\eqref{eq:conformal_ellipsoid_higherorder} up to second-order ($\kappa=1$). For speed, we prefer the quaternion representation of the set~\eqref{eq:purse} which requires only $4$ variables to represent rotations instead of $9$ (see~\prettyref{appendix:quaternions}). We use TSSOS~\cite{Wang20arXiv-cs-tssos} to write the polynomial inequality \eqref{eq:sosslem} and modify the dual problem to match~\eqref{eq:conformal_ellipsoid_higherorder}. When it is not clear from context, we use $\mathrm{SLUE}$-$\kappa$ to refer to SLUE with a $\kappa$-order relaxation. 

\section{Translational and Angular Uncertainty}
\label{sec:uncertainty_bounds}
The ellipsoid~\eqref{eq:ellipse_understandable} expresses rotation and translation uncertainty jointly. However, it is not immediate to infer rotation uncertainty from the ellipsoidal bound on joint translation and rotation matrix uncertainty. In view of \prettyref{prob:uncertainty}, we next describe a simple projection scheme to obtain separate bounding ellipsoids for translation and orientation, each interpretable as explicit uncertainty bounds.

\textbf{Translational Ellipsoid.}
To reduce \eqref{eq:ellipse_understandable} into a translation-only ellipsoid, we project $\*H$ onto its last three coordinates via orthogonal projection \cite{Karl92tr}. Let $\*P_t \triangleq \begin{bmatrix}\*0_{3\times9} & \eye_3\end{bmatrix}$. The translations $\*t\in\RR^3$ satisfying~\eqref{eq:ellipse_understandable} also satisfy the following ellipsoidal constraint:
\begin{equation}
    \label{eq:marginalized_translations}
    (\*t - \*{\bar t})^\T \*H_t(\*t - \*{\bar t}) \leq 1,
\end{equation}
where $\*H_t \triangleq \left(\*P_t\*H^{-1}\*P_t^\T\right)^{-1}$ is the ellipsoid matrix.

\textbf{Angular Ellipsoid.}
The angular ellipsoid requires more care. Similar to translations, we can project $\*H$ onto its first nine coordinates. This gives the ellipsoid:
\begin{equation}
    \label{eq:marginalized_rotations}
    (\*r - \*{\bar r})^\T \*H_r(\*r - \*{\bar r}) \leq 1,
\end{equation}
where $\*H_r \triangleq \left(\*P_r\*H^{-1}\*P_r^\T\right)^{-1}$ and $\*P_r \triangleq \begin{bmatrix}\eye_9 & \*0_{9\times3}\end{bmatrix}$. 

Eq.~\eqref{eq:marginalized_rotations} is still an \emph{implicit} rotation constraint because rotation matrices must satisfy the additional constraints $\*R\in\SO3$. Thus, we transform \eqref{eq:marginalized_rotations} into an ellipsoid over \emph{axis-angle} deviation from the rotation estimate $\*{\bar R}$. For an angle $\theta$ and axis $\'\omega\in\mathbb{S}^2$, any rotation matrix $\*R\in\SO3$ can be written as an axis-angle perturbation of $\*{\bar R}$. That is, $\*R=\*R_{\'\omega}(\theta)\*{\bar R}$, where $\*R_{\'\omega}(\theta)$ rotates by angle $\theta$ about axis $\'\omega$ according to Rodrigues' rotation formula~\cite[Ch. 6.2]{Barfoot17book}. Specifically, $\*R_{\'\omega}(\theta) = \eye_3 + \'{\hat \omega}\sin(\theta) + \'{\hat \omega}^2(1-\cos(\theta))$, where $\'{\hat \omega}$ is the skew-symmetric cross-product matrix for $\'\omega$:
\begin{equation}
    \label{eq:skew}
    \'{\hat \omega} \triangleq \begin{bmatrix}
        0 & -\omega_3 & \omega_2 \\
        \omega_3 & 0 & -\omega_1 \\
        -\omega_2 & \omega_1 & 0
    \end{bmatrix}
\end{equation}

For convenience, we represent axis-angle using the vector $\'\omega\sin(\theta)$. Assuming $\theta \leq 90^\circ$, we can recover $\'\omega$ from its direction and $\theta$ from its norm. The vector $\'\omega\sin(\theta)$ is particularly convenient because it is related to the skew part of the rotation matrix. From Rodrigues' formula, $\*R - \*R^\T = 2\sin(\theta)\'{\hat \omega}$. In vector form, this is:

\begin{equation}
    \label{eq:axang}
    2\'\omega\sin(\theta) = 
    \begin{bmatrix}
        \left[R_{\'\omega}(\theta)\right]_{3,2} - \left[R_{\'\omega}(\theta)\right]_{2,3}\\
        \left[R_{\'\omega}(\theta)\right]_{1,3} - \left[R_{\'\omega}(\theta)\right]_{3,1}\\
        \left[R_{\'\omega}(\theta)\right]_{2,1} - \left[R_{\'\omega}(\theta)\right]_{1,2}
    \end{bmatrix}\!.
\end{equation}

Using~\eqref{eq:axang}, we can project the rotation matrix ellipsoid to an ellipsoid in $\'\omega\sin(\theta)$. As an intermediate step, we rewrite~\eqref{eq:marginalized_rotations} as an ellipsoid in $\*R_{\'\omega}(\theta)$. Notice that $\*r - \*{\bar r} = \mathrm{vec}\left[(\*R_{\'\omega}(\theta) - \eye_3)\*{\bar R}\right]$. By the Kronecker identity,
\begin{equation}
    \*r - \*{\bar r} = 
    (\*{\bar R}^\T \otimes \eye_3)\mathrm{vec}(\*R_{\'\omega}(\theta) - \eye_3).
\end{equation}
The vector $\mathrm{vec}(\*R_{\'\omega}(\theta) - \eye_3)$ can then be projected to $\'\omega\sin(\theta)$ using \eqref{eq:axang} (subtracting the identity leaves the skew part unchanged). This procedure is summarized in the proposition below.

\begin{proposition}
    \label{prop:ang_bounds}
    Let $\*P_\theta$ be the projection matrix such that $\*P_\theta\mathrm{vec}(\*R_{\'\omega}(\theta)) = \eqref{eq:axang}$. Define the matrix $\*H_\theta$ as below:
    \begin{equation}
        \label{eq:Htheta}
        \*H_\theta \triangleq 4\left(\*P_\theta
        \left[(\*{\bar R}^\T \otimes \eye_3)^\T\*H_r (\*{\bar R}^\T \otimes \eye_3)\right]^{-1}
        \*P_\theta^\T\right)^{-1}\!\!.
    \end{equation}
    
    If $\theta\leq 90^\circ$, the set of axis-angle deviations ($\'\omega,\theta$) from the rotation estimate $\*{\bar R}$ satisfying~\eqref{eq:ellipse_understandable} also satisfy the following ellipsoidal constraint.
    \begin{equation}
        \label{eq:marginalized_angles}
        (\'\omega\sin(\theta))^\T \*H_\theta (\'\omega\sin(\theta)) \leq 1.
    \end{equation}
    At a given $\'\xi$ in~\eqref{eq:marginalized_angles}, $\'\omega = \'\xi / \|\'\xi\|_2$ and $\sin(\theta) = \|\'\xi\|_2$.
\end{proposition}

Restricting the angular deviation ($\theta\leq90^\circ$) is not strictly necessary. Using the quaternion rotation representation gives a similar marginalization scheme which does not require any orientation restrictions, see~\prettyref{appendix:quaternions}. 

\section{Experiments}
\label{sec:experiments}
In this section, we evaluate SLUE in three object-focused scenarios. We first describe the three datasets. Then, in \prettyref{sec:keypoints}, we show how to obtain the high-probability keypoint uncertainty bounds \eqref{eq:conformal_boundedprob} from split conformal prediction \cite{Angelopoulos24-ConformalTextbook} and provide quantitative coverage results in \prettyref{sec:coverage}. 
Our main results are in \prettyref{sec:results_ellipse}. In particular, we compute angular and translational bounds for each dataset. Compared to prior work, using SLUE to optimize directly for the bounding ellipsoid's shape yields significantly smaller uncertainty sets without sacrificing outer bound guarantees. In \prettyref{sec:runtime} we show SLUE is the fastest bounding approach despite solving also for ellipsoid shape. Further qualitative results and ablations are in \prettyref{appendix:conformal_exp}.

\begin{figure*}[htb!]
    \centering
    \subfloat[\textsf{CAST} position uncertainty]{\includegraphics[width=0.32\linewidth]{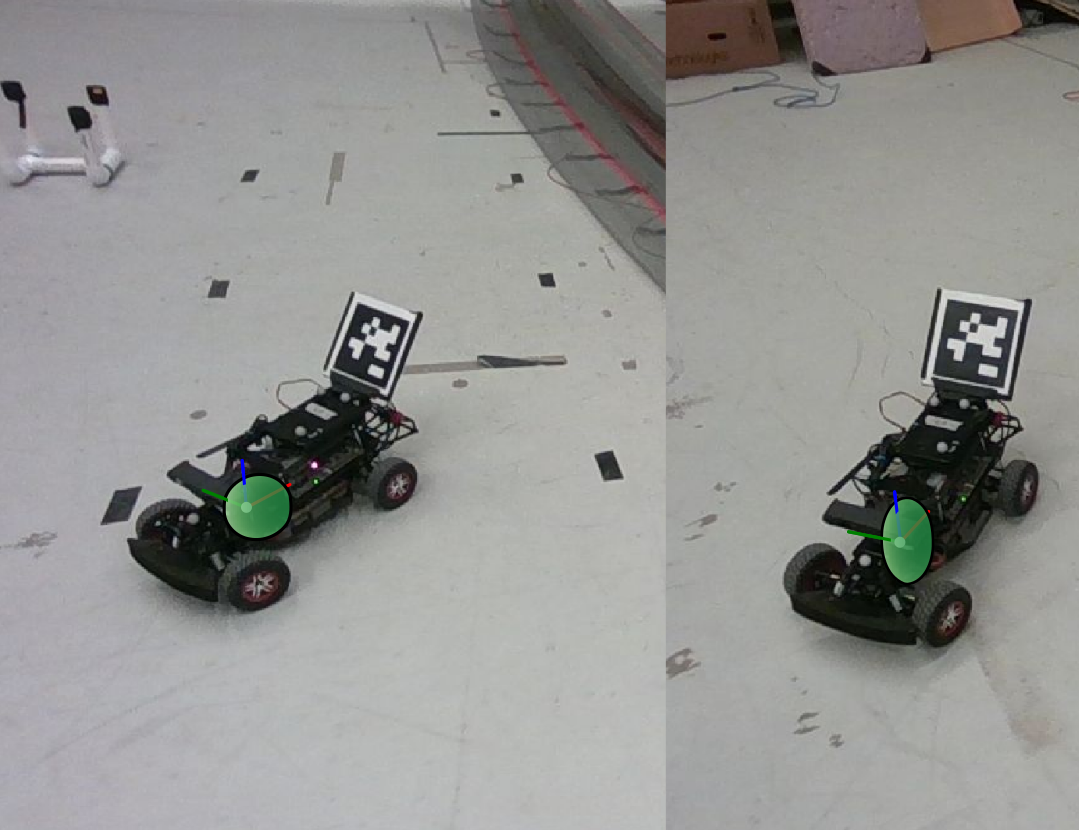}}
    \hspace{0.2cm}
    \subfloat[\textsf{LM-O} uncertainty]{\includegraphics[width=0.32\linewidth]{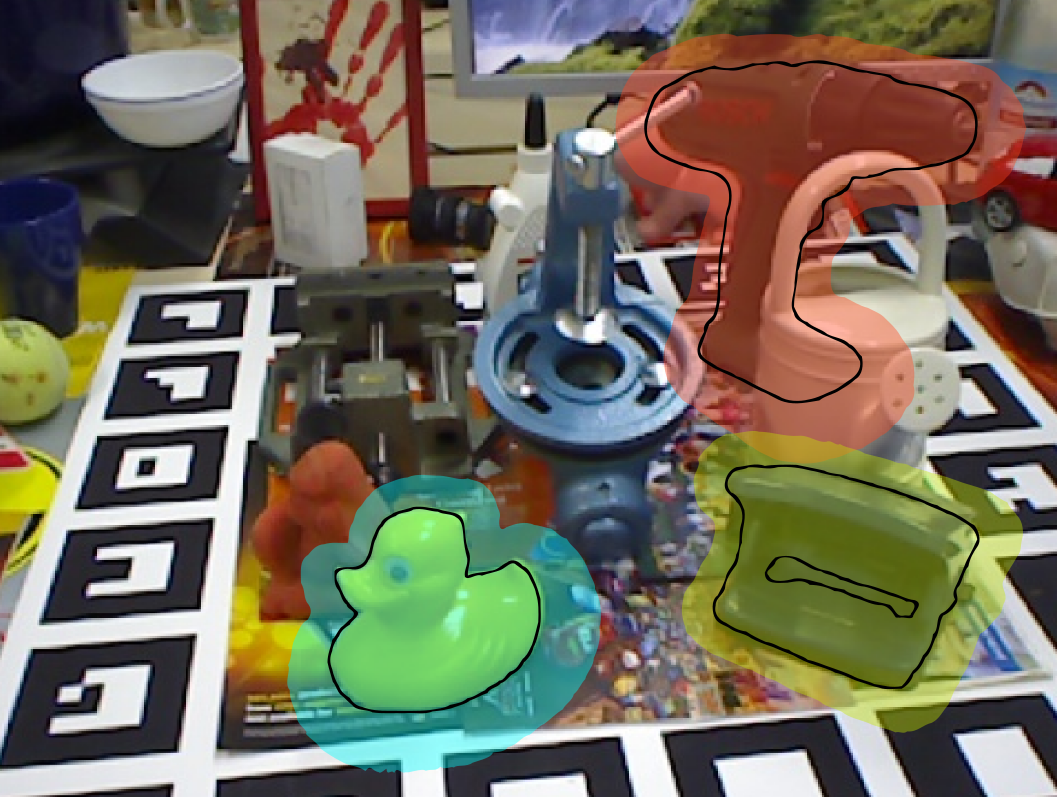}}
    \hspace{0.2cm}
    \subfloat[\textsf{YCB-V} uncertainty]{\includegraphics[width=0.32\linewidth]{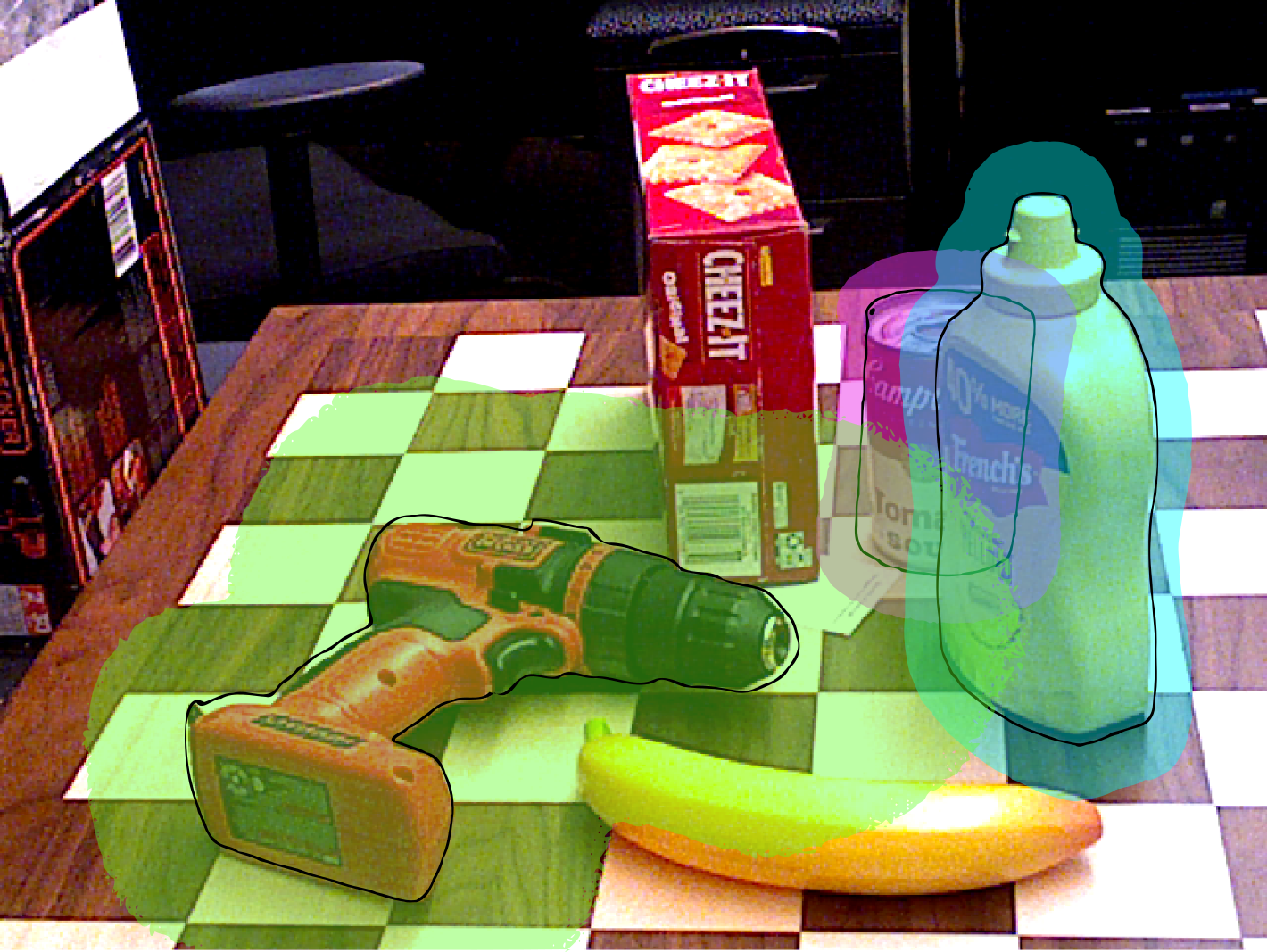}}
    \caption{\textbf{Image-plane projections of ellipsoidal pose uncertainty.} Plots show the set of possible poses in the second-order joint ellipsoidal bound for $\alpha=0.1$. Uncertainty is mostly concentrated along the optical axis. For \textsf{\small CAST}, we only show translation uncertainty about the pose estimate.}
    \label{fig:qual}
\end{figure*}

\textbf{Datasets.} We test SLUE on the LineMOD-Occlusion (\textsf{\small LM-O}) dataset~\cite{Brachmann14cvpr}, the YCB-Video (\textsf{\small YCB-V}) dataset~\cite{Xiang17rss-posecnn}, and the drone tracking scenario from~\cite{Shaikewitz24ral-CAST} (\textsf{\small CAST}). \textsf{\small LM-O} and \textsf{\small YCB-V} are object-focused datasets, with RGB images of $8$ and $21$ tabletop objects, respectively. For \textsf{\small LM-O} we omit the eggbox object following~\cite{Yang23arxiv-ransag}. For \textsf{\small YCB-V}, we test on the BOP subset of images~\cite{Hodan20eccvw-BOPChallenge} and include objects $1$ to $15$ (except for object $13$, the bowl). We omit object $13$ due to symmetry and objects $16$ through $21$ because there are few examples in the BOP subset. For a robotics setting, the drone tracking scenario \textsf{\small CAST} contains RGB images of a remote-controlled racecar taken from an aerial quadcopter. We use the entire image sequence which includes several laps of the racecar.  See \prettyref{fig:qual} for example images.

\subsection{Keypoint Bounds via Conformal Prediction}
\label{sec:keypoints}
\textbf{Keypoint Detection.} Given an RGB image, we use a learned front-end to extract the pixel positions of target object keypoints. In the \textsf{\small LM-O} and \textsf{\small YCB-V} datasets, we use the heatmap-based convolutional neural network from~\cite{Schmeckpeper22arxiv-singleRGBpose} for keypoint detection. This network was trained on the BOP synthetic image split~\cite{Hodan18eccv-BOPBenchmark}. Like~\cite{Yang23arxiv-ransag}, we take the pixel with highest confidence in each heatmap as detection. For \textsf{\small CAST}, we use the author-provided keypoints, which come from a ResNet~\cite{He16cvpr-ResNet} detector trained on real-world images~\cite{Shaikewitz24ral-CAST}.

\textbf{Conformal Prediction.} We obtain keypoint uncertainty bounds, the input to SLUE through eq. \eqref{eq:conformal_boundedprob}, using \emph{split conformal prediction} \cite{Shafer08jmlr-TutorialConformal,Angelopoulos24-ConformalTextbook}. Conformal prediction provides high-probability error bounds through a calibration process. During calibration, we test keypoint detection on a subset of images (separate from training or testing data) where the ground truth keypoint positions are known (\eg from human annotation). Under the assumption of \emph{exchangeability} (a slight relaxation of independent and identically distributed, see \cite{Angelopoulos24-ConformalTextbook}), performance on the calibration data predicts performance on test data.

We now describe our calibration routine. For each keypoint $\*y_i$ we compute uncertainty bounds $r_i(\alpha)$ at fixed confidence $\alpha$. For \textsf{\small LM-O} we calibrate on the BOP image subset (following~\cite{Yang23arxiv-ransag}), which contains $200$ images selected from the full dataset. For \textsf{\small CAST}, we calibrate on a randomly selected subset of $200$ images in the video sequence. In both datasets, we remove the calibration images from our evaluation set to ensure exchangeability with the test images. In \textsf{\small YCB-V}, we instead calibrate on $2000$ \emph{synthetic} images of the objects freshly generated with BlenderProc~\cite{Denninger23joss-blenderproc}. Calibrating on synthetic images may violate exchangeability (the calibration images are no longer randomly drawn from the test data), but is more realistic for a real-time system.

Given an RGB image, let $\*y_i$ denote the detected pixel location of keypoint $i$ and let $c_i\in[0,1]$ be its associated confidence (for \textsf{\small CAST}, the network confidence is always $1$). Using \eqref{eq:conformal_meas}, the ground truth pixel coordinate is $\*z_i \triangleq {\*K(\*R_\mathrm{gt}\*b_i + \*t_\mathrm{gt})}/({\*{\hat{e}}_3\cdot(\*R_\mathrm{gt}\*b_i + \*t_\mathrm{gt})})$. We measure performance during conformal calibration using the following confidence-weighted distance metric:
\begin{equation}
    \label{eq:calibration}
    s((\*y_i, c_i), \*z_i) = c_i\|\*y_i - \*z_i\|_p,
\end{equation}
where $p=2$ or $p=\infty$.

The keypoint uncertainty bound $r_i(\alpha)$ (visualized in \prettyref{fig:conformal_calibration}) is given by an adjusted $1-\alpha$ quantile of calibration distance scores. Specifically, let $\tilde{r}_i(\alpha)$ be the adjusted $(1-\alpha)(1+1/n)$ quantile of the scores of keypoint $i$ on all the calibration images, where $n$ is the number of calibration images picturing keypoint $i$. Then, $r_i(\alpha) \triangleq {\tilde{r}_i(\alpha)}/{c_i}$. By~\cite[Thm. 3.2]{Angelopoulos24-ConformalTextbook}, the high-probability error bound~\eqref{eq:conformal_boundedprob} holds when calibration and evaluation data are exchangeable.

\begin{table}[tb]
    \centering
    \caption{Mean Empirical Coverage (\%) of Ground Truth Keypoints (Kpts), Pose Uncertainty Constraint Sets~\eqref{eq:purse}, and SLUE Ellipsoids
}
    \label{tab:calibration}
    \adjustbox{width=\linewidth}
    {\begin{tabular}{rcccccc}
        \toprule
        & \multicolumn{3}{c}{$\alpha=0.1$} & \multicolumn{3}{c}{$\alpha=0.4$}\\
        \cmidrule(lr){2-4}\cmidrule(lr){5-7}
        & \multicolumn{1}{c}{Kpts} & \multicolumn{1}{c}{\eqref{eq:purse}} & \multicolumn{1}{c}{SLUE} & 
        \multicolumn{1}{c}{Kpts} & \multicolumn{1}{c}{\eqref{eq:purse}} & \multicolumn{1}{c}{SLUE}\\
        \midrule
        \textsf{CAST}
        & 90.9 & 56.5 & 86.5 & 58.3 & 5.6 & 19.5
        \\
        \textsf{LM-O}
        & 90.7 & 68.5 & 91.1 & 60.2 & 15.1 & 45.8
        \\
        \textsf{YCB-V}
        & 93.0 & 72.9 & 90.9 & 47.4 & 5.4 & 20.2
        \\
        \bottomrule
    \end{tabular}}
\end{table}

\subsection{Empirical Coverage Results.}
\label{sec:coverage}
To evaluate our high-probability uncertainty guarantees, we compute keypoint and pose \emph{coverage}. Coverage is defined as the percentage of keypoints (or poses) with uncertainty sets that contain the ground truth keypoint (or pose). In \prettyref{tab:calibration} we provide empirical coverage of ground truth keypoints in each keypoint uncertainty set at two selected confidences: $\alpha=0.1$ and $\alpha=0.4$. We also give coverage of the ground truth pose in both the pose uncertainty constraint set~\eqref{eq:purse} and the first-order bounding ellipsoid from SLUE~\eqref{eq:ellipse_understandable} (more details in~\prettyref{sec:results_ellipse}). The former roughly quantifies the probability $\beta$ in~\prettyref{prop:pus}. We caution that coverage results are subject to several confounding factors including the inaccuracy of hand-labeled ground truth poses given by the datasets.

We report the mean coverage for each dataset in \prettyref{tab:calibration}. For the datasets with real-world calibration (\textsf{\small LM-O} and \textsf{\small CAST}), the empirical coverage nearly matches the $1-\alpha$ guarantee. For \textsf{\small YCB-V}, which uses synthetic calibration, the $\alpha=0.1$ bounds slightly over cover while the $\alpha=0.4$ bounds under cover. We caution that \textsf{\small LM-O} and \textsf{\small YCB-V} ground truth poses are imprecise~\cite{Yang23arxiv-ransag}. Small errors in pose are magnified when propagated to keypoints.

For pose coverage, \prettyref{tab:calibration} shows the pose uncertainty constraint set suffers a $30-45\%$ drop in coverage compared to keypoints. There are $8-11$ keypoints per object for \textsf{\small LM-O}, and $7$ per object for \textsf{\small CAST}. This renders the worst-case bound $1-N\alpha$ (\prettyref{prop:pus}) near useless. The pose coverage is slightly better than the independence case, suggesting our bounds benefit from some positive correlation between the keypoints.

Lastly, the ellipsoidal uncertainty set at order $1$ (SLUE) is not overly conservative. For $\alpha=0.1$, it achieves around 90\% coverage (this is a coincidence, we can only guarantee the ellipsoid is more conservative than~\eqref{eq:purse}). The ellipsoid is a little tighter for $\alpha=0.4$. We provide quantitative and qualitative uncertainty bound results in the next section.

\subsection{Pose Uncertainty Bounds}
\label{sec:results_ellipse}
In this section, we compute ellipsoidal pose uncertainty bounds for the set~\eqref{eq:purse} using SLUE. In particular, we compute first and second-order bounding ellipsoids (\prettyref{sec:uncertainty}) using the pose estimates from PnP as center. After solving for these ellipsoids we marginalize to ellipsoids in translation and rotation space via projection as outlined in~\prettyref{sec:uncertainty_bounds}. We show quantitative and qualitative results for $\alpha=0.1$ and leave qualitative results for $\alpha=0.4$ to~\prettyref{appendix:conformal_exp}.

\begin{figure*}[htb!]
    \centering
    \subfloat[Translation CDF for \textsf{CAST}.]{\includegraphics[width=0.308\linewidth]{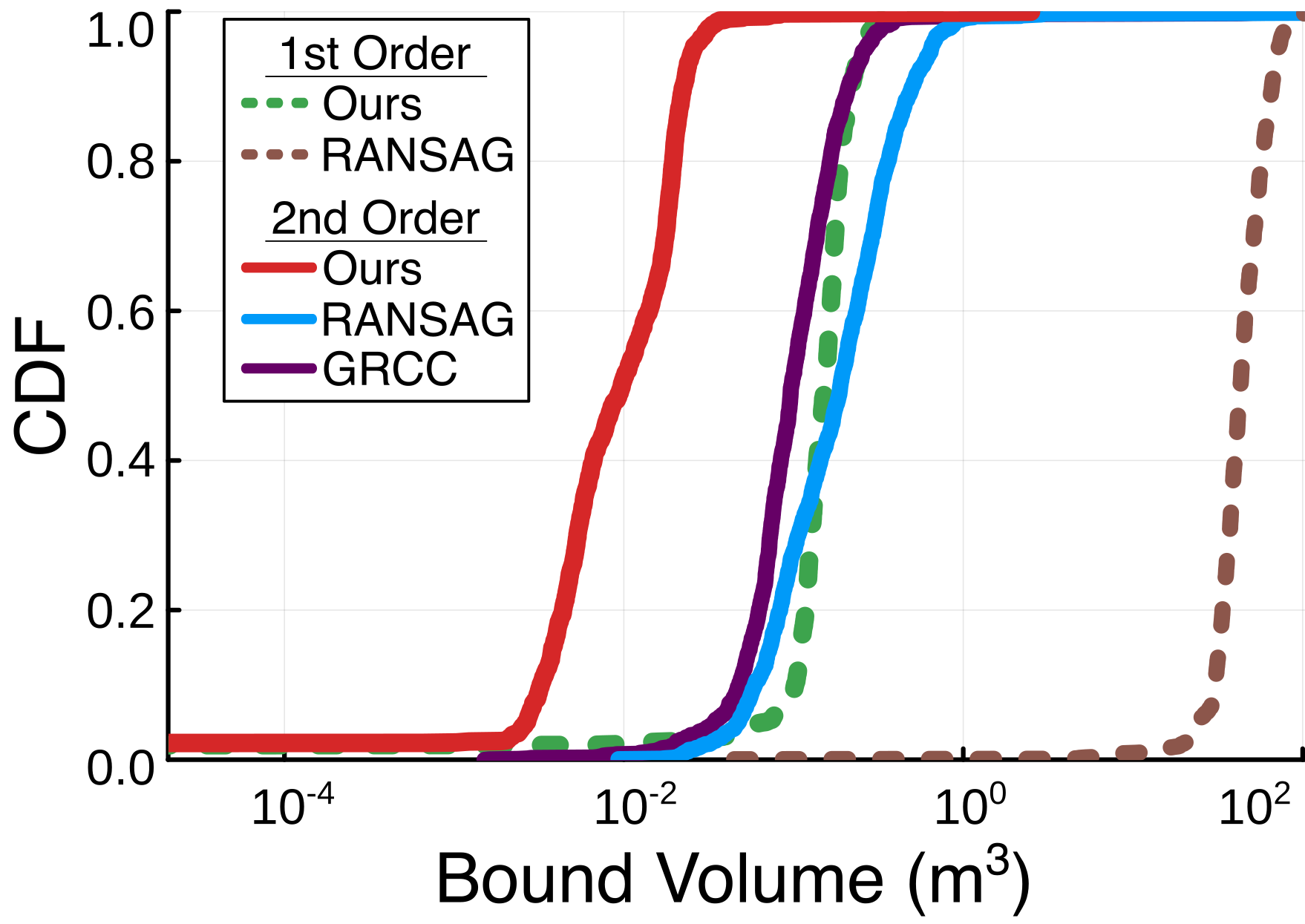}}
    \hspace{0.2cm}
    \subfloat[Translation CDF for \textsf{LM-O}.]{\includegraphics[width=0.32\linewidth]{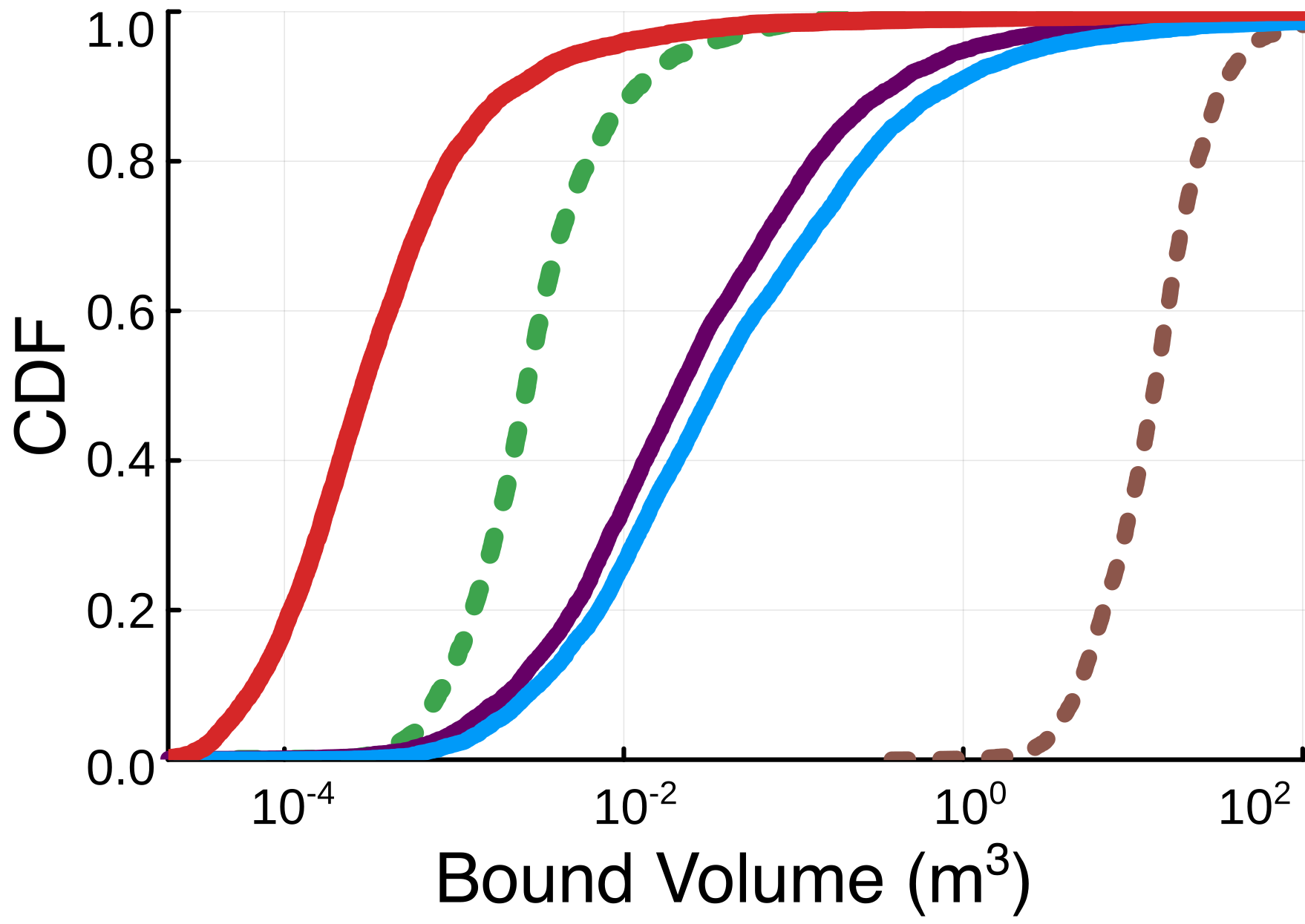}}
    \hspace{0.2cm}
    \subfloat[Translation CDF for \textsf{YCB-V}.]{\includegraphics[width=0.32\linewidth]{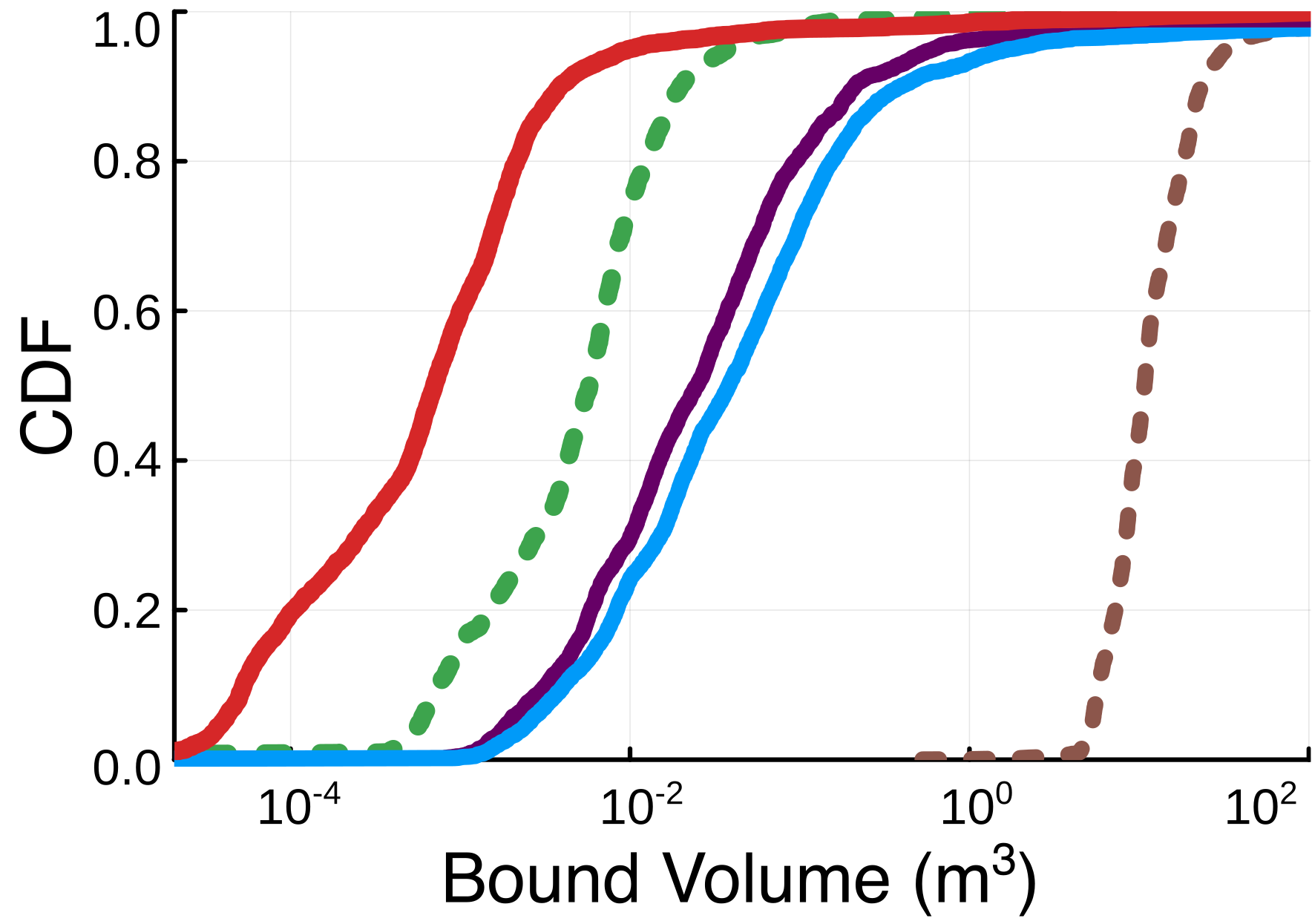}}\\
    \subfloat[Angular CDF for \textsf{CAST}.]{\includegraphics[width=0.32\linewidth]{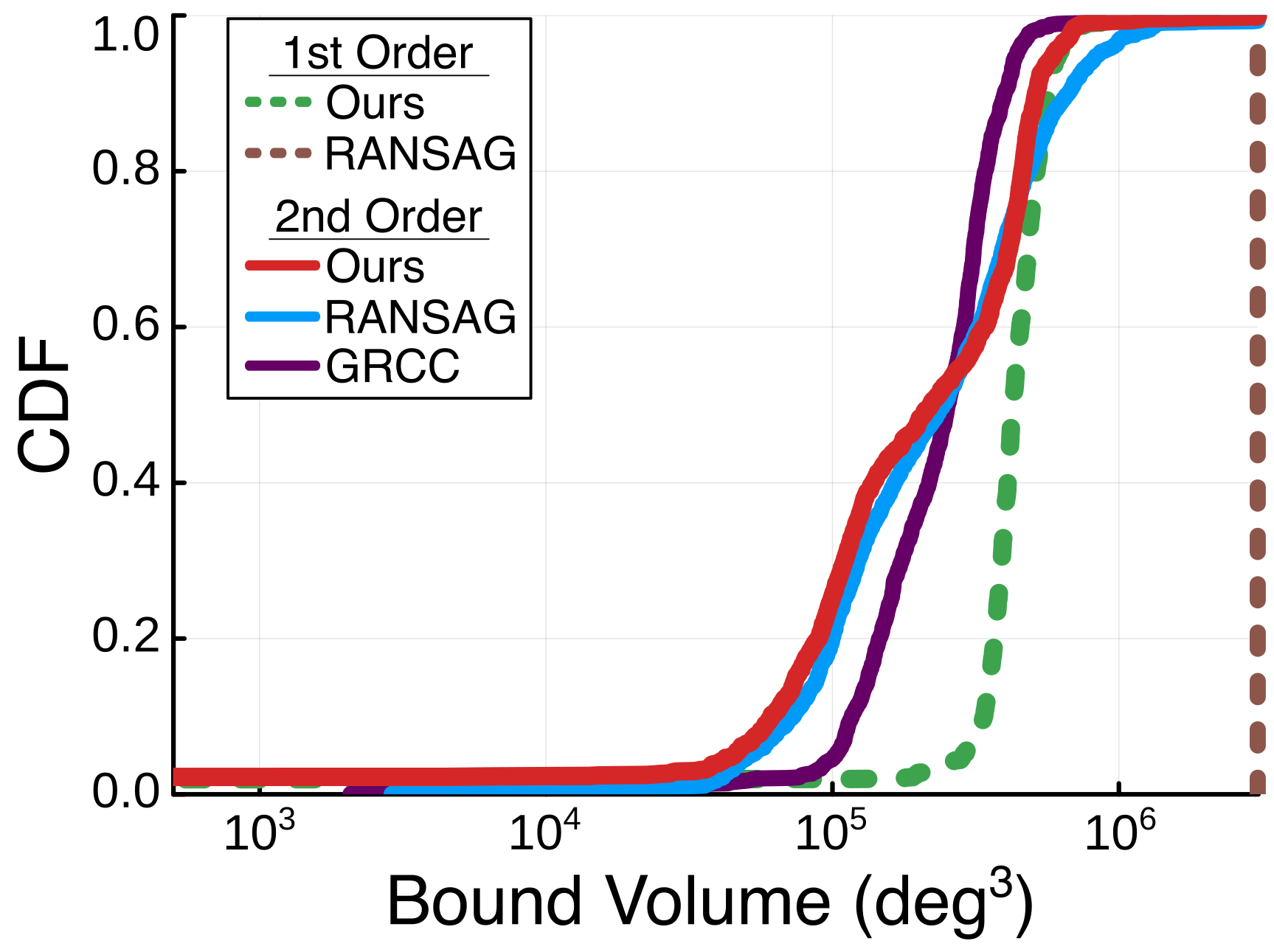}}
    \hspace{0.2cm}
    \subfloat[Angular CDF for \textsf{LM-O}.]{\includegraphics[width=0.32\linewidth]{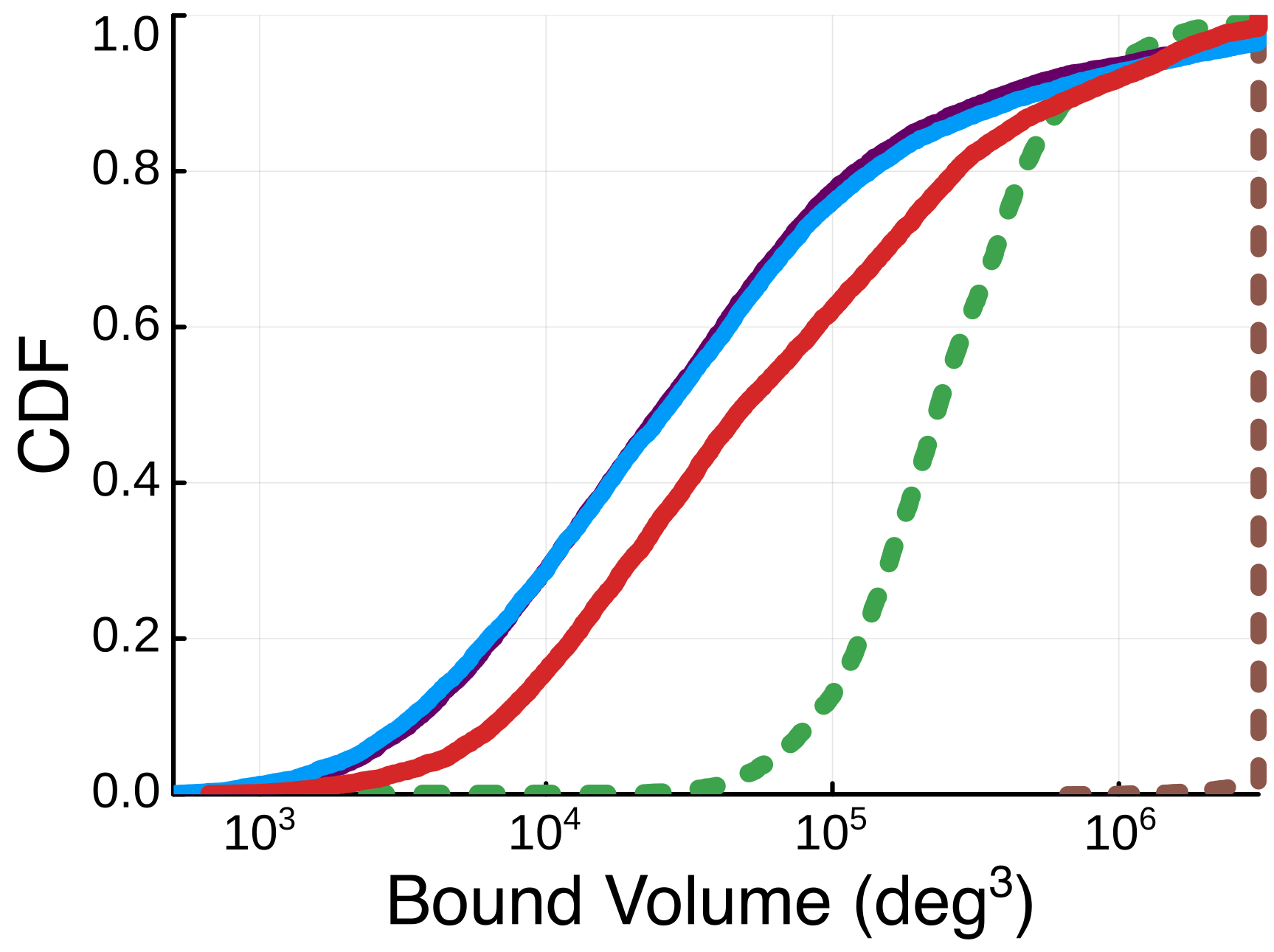}}
    \hspace{0.2cm}
    \subfloat[Angular CDF for \textsf{YCB-V}.]{\includegraphics[width=0.32\linewidth]{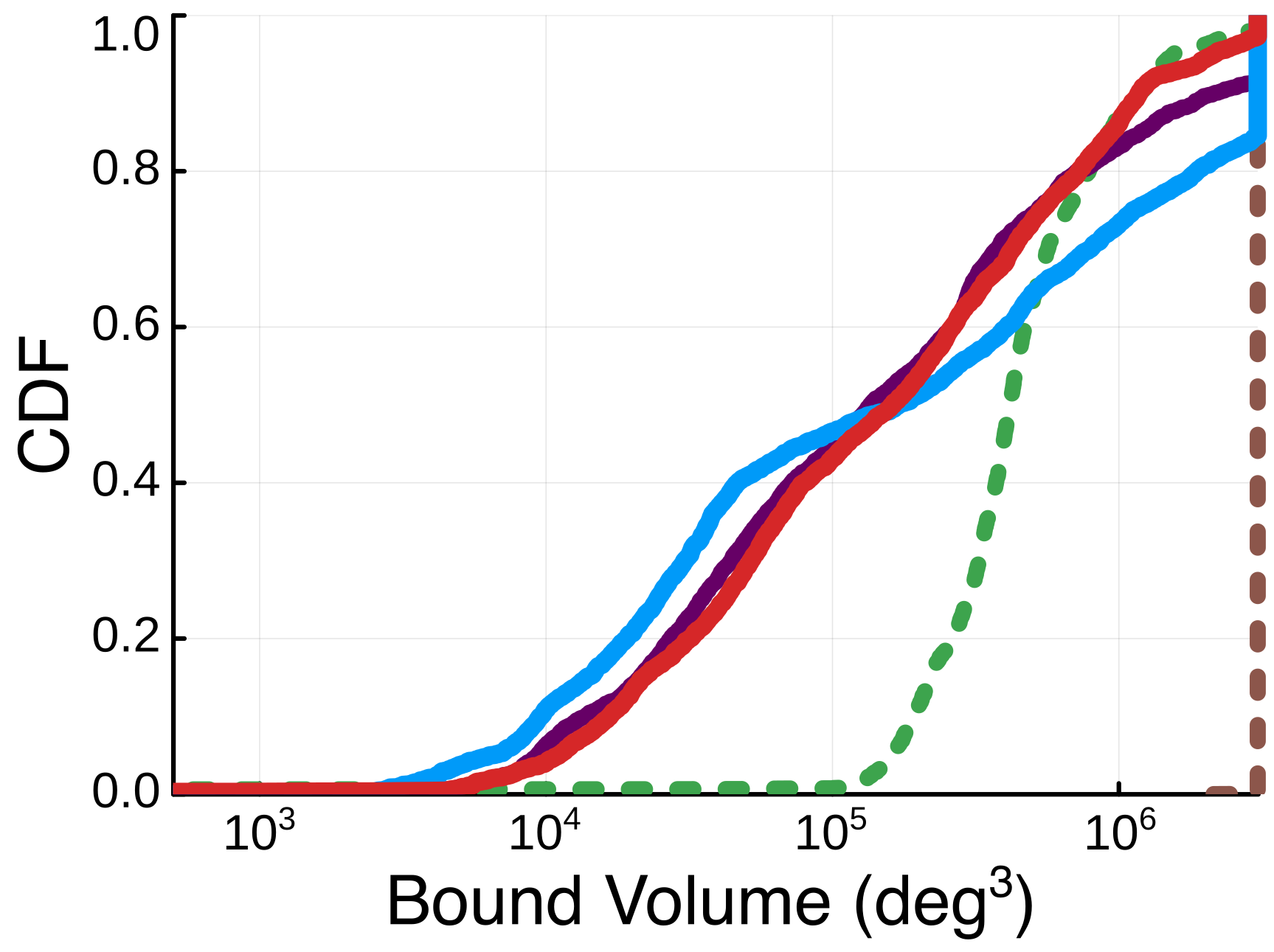}}
    \caption{\textbf{Volume of Translational and Angular Uncertainty Ellipsoids.} We plot the cumulative distribution function (CDF) of each approach at first and second-order. Our translation and angular bounds are significantly tighter at first-order, while only the translation bounds are conclusively tighter for a second-order relaxation.}
    \label{fig:uncertainty_bounds}
\end{figure*}

\textbf{Baselines.} We compare SLUE against RANSAG~\cite{Yang23arxiv-ransag} and GRCC~\cite{Tang24l4dc-setMembership}. Both methods solve relaxations of the bounding ellipsoid problem where the ellipsoid shape is fixed and optimize only for its scale, solving the rotation and translation problems separately. They also use $2$-norm keypoint uncertainty, which does not affect coverage guarantees. Following~\cite{Tang24l4dc-setMembership}, we use a spherical bound and solve for its scale. We test RANSAG with a first and second-order relaxation and GRCC with only a second-order relaxation (GRCC requires at least second-order). We use a custom Julia implementation of RANSAG and rely on the open-source MATLAB code for GRCC. Results exclude image frames where one or more methods reports failure, leaving $6070/6190$ images for \textsf{\small LM-O}, $3153/3256$ images for \textsf{\small YCB-V}, and $1587/1691$ images for \textsf{\small CAST}.

\textbf{Quantitative Results.} \prettyref{fig:uncertainty_bounds} shows the cumulative distribution function (CDF) of the volume of the translation and rotation ellipsoids for each method in \prettyref{fig:uncertainty_bounds}. The translation volume is $(4\pi/3) \log\det(\*H_t)$. For the orientation ellipsoid, the volume is $(4\pi/3) \theta_1 \theta_2\theta_3$, where $\theta_i$ is the $i$th principal axis of the ellipsoid $\*H_\theta$, in degrees, clipped above $90^\circ$ to be consistent with our assumptions. Each method computes an outer approximation of pose uncertainty at a fixed confidence, so a smaller volume generally indicates a better approximation.

As \prettyref{fig:uncertainty_bounds} shows, our approach (SLUE) excels at producing tight translation bounds. In particular, the first and second-order S-lemma ellipsoids are smaller than their baselines by multiple orders of magnitude. Interestingly, SLUE-$1$ is often much smaller than even the second-order ellipsoid generated by RANSAG or GRCC. The key reason for this improvement is optimizing directly for the shape of the uncertainty ellipsoid.

\begin{figure*}[htb!]
    \centering
    \subfloat[Translation uncertainty ellipsoids.]{\includegraphics[width=0.27\linewidth]{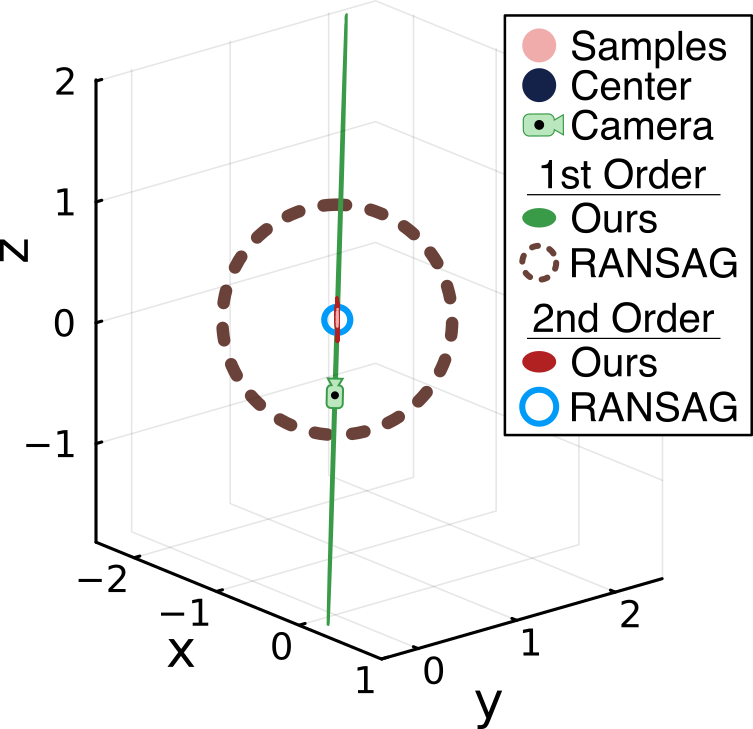}}
    \hspace{0.5cm}
    \subfloat[Close view of second-order translation uncertainty ellipsoids.]{\includegraphics[width=0.27\linewidth]{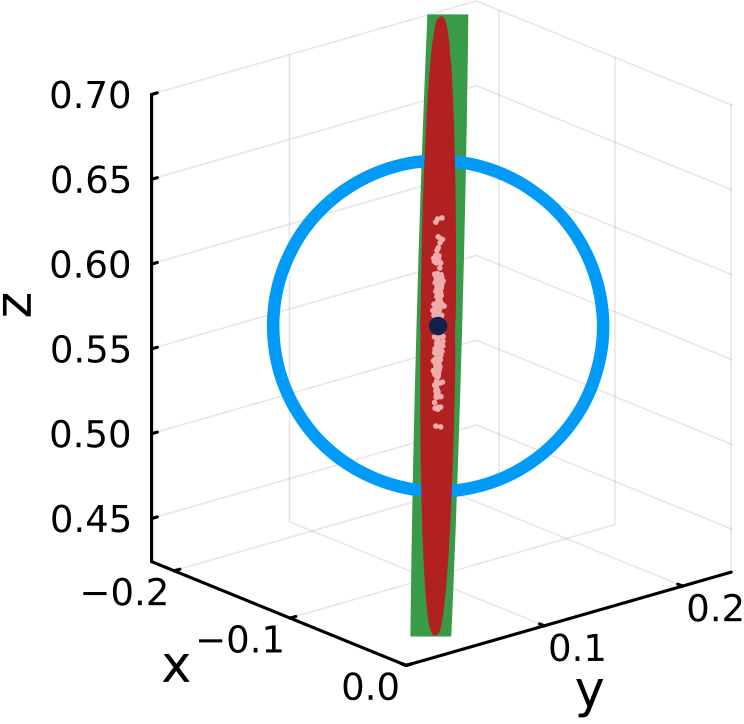}}
    \hspace{0.5cm}
    \subfloat[Angular uncertainty ellipsoids.]{\includegraphics[width=0.32\linewidth]{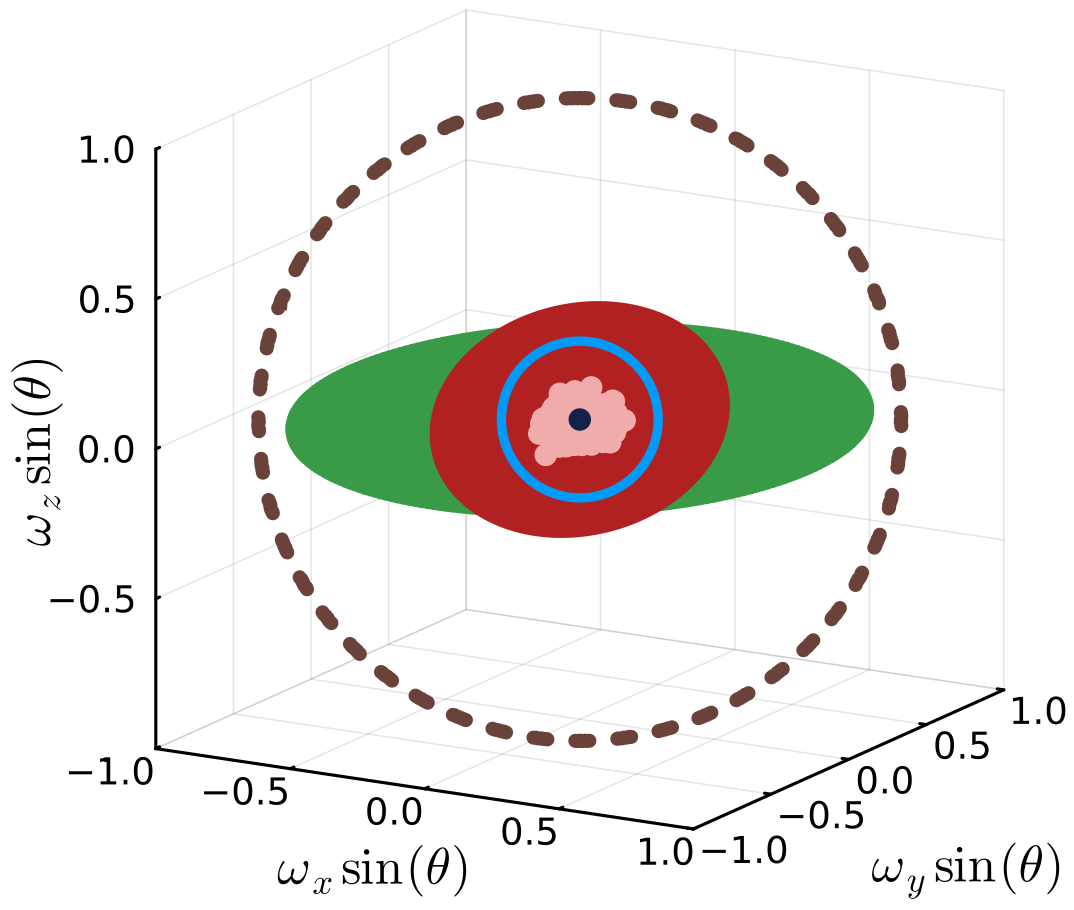}}
    \caption{\textbf{Qualitative bounding sets in rotation and translation space.} Projections of the first and second-order uncertainty ellipsoids generated by SLUE and RANSAG~\cite{Yang23arxiv-ransag} for the duck object on \textsf{\small LM-O} ($\alpha = 0.1$). SLUE optimizes directly for the shape of the ellipsoid, leading to expressive uncertainty sets. The ellipsoids enclose all samples from~\eqref{eq:purse}.}
    \label{fig:ellipses_10}
\end{figure*}

In the angular case, SLUE significantly outperforms RANSAG for a first-order relaxation but its advantages are less clear at second-order. A confounding variable is the choice of keypoint uncertainty set norm. RANSAG and GRCC use $p=2$, while SLUE uses $p=\infty$ for its computational benefits. Although the uncertainty set confidences are the same, the $p=\infty$ case may allow more extreme angular changes. Another difference is that RANSAG and GRCC optimize for an independent rotation bound instead of solving a single problem for a joint ellipsoid. Nevertheless, SLUE-$2$ is competitive across the datasets. A key strength of our approach is that it only requires solving a single optimization problem to obtain an angular and translational bounding ellipsoid. We quantify this computational advantage in \prettyref{sec:runtime}.

\textbf{Qualitative Results.}
To complement the CDF plots, we provide qualitative plots of the ellipsoids generated by SLUE and RANSAG. \prettyref{fig:ellipses_10} shows translational and angular uncertainty ellipsoids at first and second-order. The ellipsoids are plotted for duck object in frame $342$ of the \textsf{\small LM-O} dataset, but they are representative examples. \prettyref{fig:qual} shows the image-plane projection of the second-order ellipsoid bounds at this frame, along with representative frames for the other two datasets.

\prettyref{fig:ellipses_10}a captures the phenomenon of increased translation uncertainty along the optical axis of the camera. The first-order ellipsoid is clearly very conservative along the optical axis (it expands to include the origin), but remains relatively tight along orthogonal axes. The second-order ellipsoid (\prettyref{fig:ellipses_10}b) is small enough to serve as a useful translation bound, and clearly contains all the sampled points. Both SLUE ellipsoids are far more expressive than the RANSAG ellipsoids and accurately capture the uncertainty due to scale ambiguity. In contrast, the qualitative rotation ellipsoids (\prettyref{fig:ellipses_10}c) do not clearly show an elongated axis from the monocular setup. In this example, the first-order RANSAG orientation bound is greater than $90$ degrees, forming the unit sphere in our axis-angle coordinates. While the second-order SLUE ellipsoid is smaller, the first-order SLUE ellipsoid is slightly tighter along the $z$ axis. In general, the SLUE angular bounds are quite conservative relative to the samples and second-order RANSAG. This is an artifact of increased rotation uncertainty from using $\infty$-norm keypoint sets and the difficultly of bounding quadratic equality constraints with an ellipsoid.

For a more interpretable perspective, \prettyref{fig:qual} shows image-plane projections of poses on the surface of the joint second-order uncertainty ellipsoid. We show frames $568$ and $1000$ for \textsf{\small CAST}, frame $342$ for \textsf{\small LM-O}, and frame $501125$ for \textsf{\small YCB-V}. Due to reprojection, the uncertainty bounds are dominated by the pose with translation closest to the camera. The uncertainty bounds are relatively tight and a much richer description of pose than a single estimate. The image-plane bounds benefit from the joint uncertainty representation, which maintains correlation between translation and orientation uncertainty. For \textsf{\small CAST} we only plot translation uncertainty (no 3D model is available). The translation bounds are relatively small. We show additional results and ablations in \prettyref{appendix:conformal_exp}.

\subsection{Runtime Breakdown}
\label{sec:runtime}
We conclude our experiments with a comparison of runtimes on \textsf{\small LM-O}. In~\prettyref{tab:runtimes}, the keypoint detector is implemented in PyTorch and runs on an M3 Mac. Our method (S-Lemma Bounds) and RANSAG are implemented in Julia, while GRCC uses the released MATLAB implementation. All methods are running on a single CPU thread. For our method we omit the translation and rotation marginalization (which are simple linear algebra operations and have marginal effect on runtime). RANSAG and GRCC compute translation and rotation bounds separately;~\prettyref{tab:runtimes} adds the two times. The GRCC runtime is inflated due to its MATLAB implementation; we indicate this with an asterisk. Runtimes are reported for confidence $\alpha=0.1$.

\begin{table}[tb]
    \centering
    \caption{Comparison of Mean Runtimes on \textsf{\small LM-O}}
    \label{tab:runtimes}
    \begin{tabular}{rcc}
        \toprule
        Relaxation order & $\kappa=1$ (ms) & $\kappa=2$ (ms) \\
        \midrule
        Keypoint detection & 48.8 & 48.8\\
        SLUE & 39.0 & 548\\
        RANSAG~\cite{Yang23arxiv-ransag} & 238 & 3,671\\
        GRCC~\cite{Tang24l4dc-setMembership} & - & 36,531*\\
        \bottomrule
    \end{tabular}
\end{table}

Our method is significantly faster than prior work, running faster than even keypoint detection at first-order. We benefit from a simpler infinity-norm pose uncertainty set, solving only a single optimization problem, and from the more flexible ellipsoid volume objective. The computational benefits are even more extreme at second-order, where the quaternion implementation (\prettyref{appendix:quaternions}) of SLUE reduces the number of decision variables.

 \section{Conclusion}
We presented SLUE: a hierarchy of algorithms to generate ellipsoidal object pose uncertainty bounds which are statistically guaranteed to cover the true object pose with high probability. SLUE may be used to add statistically rigorous uncertainty to any pose estimator. It relies only on the assumption of keypoint noise which is bounded with high-probability. While the set of object pose constraints is non-convex, we use a generalization of the S-lemma to formulate a convex relaxation for the minimum volume ellipsoid bounding the pose uncertainty set. This relaxation admits a sum-of-squares hierarchy guaranteed to converge to the true minimum volume bounding ellipsoid. From the joint ellipsoid bounds we use a simple projection scheme to bound translation and axis-angle error separately. Real-world evaluations show our method achieves tighter translation bounds and competitive angular bounds while being significantly faster than prior work.

\textbf{Limitations and Future Work.} While our work provides a rigorous and practical strategy for obtaining tight pose uncertainty bounds from keypoint error bounds, we do not make use of the high-probability nature of the keypoint bounds. Future work should consider that some keypoints may be outside of their uncertainty sets, which leads to a combinatorial multiple-testing problem. This is already well-understood in the maximum likelihood literature as outlier-robust estimation and could be applied to obtain better pose estimates and uncertainty. For some computational penalty, our S-lemma relaxation could be modified to estimate the ellipsoid center jointly with its shape~\cite{Nie05optim-MinimalEnclosingEllipsoid}. It would also be useful to capture orientation uncertainty due to, for example, object symmetry. This would require a more expressive primitive or keypoint conformal score. From a more practical perspective, robotics applications often see sequential frames picturing the same object; in the tracking setting, the fast first-order uncertainty could be fused or incorporated into an active perception scheme.

\appendices
\renewcommand{\theequation}{A\arabic{equation}}
\renewcommand{\thetheorem}{A\arabic{theorem}}

\section{$2$-Norm Pose Uncertainty Constraint Set}
\label{appendix:2norm}
For completeness and alignment with \cite{Yang23arxiv-ransag}, we give the pose uncertainty constraint set when measurement noise is bounded in $2$-norm with high probability. As before, the measurement noise bounds give the following reprojection constraint:
\begin{equation}
    \label{eq:yminusgt2}
    \left\|\*y_i - \frac{\*K(\*R\*b_i + \*t)}{\*{\hat{e}}_3\cdot(\*R\*b_i + \*t)}\right\|_2 = \|\'\eps_i\|_2 \leq r_i.
\end{equation}

The chirality constraints are the same as~\eqref{eq:frontofcamera}. For the backprojection constraints, square the inequality~\eqref{eq:yminusgt2} and multiply through by the depth. The following inequality holds with probability at least $1-\alpha_i$:
\begin{equation}
    \label{eq:purse2_setconstraints}
    \tag{$\text{BP}_2$}
    \left\|
        (\*y_i\*{\hat{e}}_3^\T - \*K)(\*R\*b_i + \*t) 
    \right\|^2_2
    \leq
    (r_i\*{\hat{e}}_3^\T(\*R\*b_i + \*t))^2\!.\!
\end{equation}

The pose uncertainty constraints are simply the set of \eqref{eq:purse2_setconstraints} and \eqref{eq:frontofcamera} for each keypoint. Note that there is only one backprojection inequality for each keypoint, but it is \emph{quadratic} in $\*t$ and $\*R$. The $2$-norm uncertainty may be used in place of $\infty$-norm uncertainty for the remainder of SLUE.

\section{Explicit Forms for Constraint Matrices}
\label{appendix:roteq}
This section gives explicit forms for the constraint matrices $\*Q_k$ and $\*A_i$ used in~\prettyref{sec:uncertainty}. We write the constraint matrices in sparse form, specifying only the nonzero elements. The notation $(i,j,l)$ means $Q_{i,j}=Q_{j,i}=l$. Let $\*R\in\SO3$ and denote its columns by $\*r_i$, $i\in[3]$. Recall that $\*x \triangleq [1, \*r^\T \*t^\T]^\T$.

We first write the equality constraint matrices $\*Q_k$, which enforce $\*R\in\SO3$. Three matrices enforce constrain the columns of $\*R$ to unit 2-norm ($\|\*r_i\|^2 = 1$):
\begin{equation}
    \begin{array}{rl}
        \*Q_1:& (2,2,1), (3,3,1), ( 4,4,1), (1,1,-1),\\
        \*Q_2:& (5,2,1), (6,3,1), ( 7,4,1), (1,1,-1),\\
        \*Q_3:& (8,2,1), (9,3,1), (10,4,1), (1,1,-1).
    \end{array}
\end{equation}

Three matrices enforce orthogonal columns:
\begin{equation}
    \begin{array}{crl}
        \*r_1\cdot\*r_2 = 0 \Rightarrow & \hspace{-5pt} \*Q_4 \hspace{-8pt} &: (2,5,1), (3,6,1), (4,7,1),\\
        \*r_1\cdot\*r_3 = 0 \Rightarrow & \hspace{-5pt} \*Q_5 \hspace{-8pt} &: (2,8,1), (3,9,1), (4,10,1),\\
        \*r_3\cdot\*r_2 = 0 \Rightarrow & \hspace{-5pt} \*Q_6 \hspace{-8pt} &: (5,8,1), (6,9,1), (7,10,1).
    \end{array}
\end{equation}

The final $9$ equality matrices enforce the right hand rule cross products between the columns of $\*R$, element-wise. The constraint $\*r_1\times\*r_2 = \*r_3$ gives:
\begin{equation}
    \begin{array}{rl}
            \*Q_7:& (3,7,1), (4,6,-1), (1,8,-1),\\
            \*Q_8:& (4,5,1), (2,7,-1), (1,9,-1),\\
            \*Q_9:& (2,6,1), (3,5,-1), (1,10,-1).
    \end{array}
\end{equation}
The constraint $\*r_2\times\*r_3 = \*r_1 $ gives:
\begin{equation}
    \begin{array}{rl}
            \*Q_{10}:& (6,10,1), (7,9,-1), (1,2,-1),\\
            \*Q_{11}:& (7,8,1), (5,10,-1), (1,3,-1),\\
            \*Q_{12}:& (5,9,1), (6,8,-1), (1,4,-1).
    \end{array}
\end{equation}
Lastly, $\*r_3\times\*r_1 = \*r_2$ gives:
\begin{equation}
    \begin{array}{rl}
            \*Q_{13}:& (9,4,1), (10,3,-1), (1,5,-1),\\
            \*Q_{14}:& (10,2,1), (8,4,-1), (1,6,-1),\\
            \*Q_{15}:& (8,3,1), (9,2,-1), (1,7,-1).
    \end{array}
\end{equation}

In~\prettyref{eq:purse_qcqp}, the inequality constraint matrices $\*A_i$ correspond to the backprojection and chirality constraints. The first $N$ matrices capture chirality. Using the Kronecker identity, the chirality constraint is:
\begin{equation}
    \label{eq:focwhynot}
    \eqref{eq:frontofcamera}_i \iff \ez^\T ((\*b_i^\T \otimes \eye_3)\*r + \*t) > 0.
\end{equation}

Noting all terms are linear, \eqref{eq:focwhynot} is:
\begin{equation}
    \eqref{eq:frontofcamera}_i \iff \begin{bmatrix}
        \ez^\T(\*b_i^\T \otimes \eye_3) & \ez^\T
    \end{bmatrix}
    \begin{bmatrix}
        \*r \\ \*t
    \end{bmatrix} > 0.
\end{equation}

This expression gives the first row and column of $\*A_i$, which correspond to multiplying $\*r$ and $\*t$ by $1$. That is,
\begin{equation}
    \label{eq:chiralitymatrix}
    \*A_i = 
    -
    \begin{bmatrix}
    0 & \ez^\T(\*b_i^\T \otimes \eye_3) & \ez^\T\\
    (\*b_i \otimes \eye_3)\ez & \*0 & \*0\\
    \ez & \*0 & \*0
    \end{bmatrix}\!.
\end{equation}

The backprojection constraints are also rewritten using the Kronecker identity. For each keypoint $i$, let:
\begin{equation}
    \*d_{ij}^\pm \triangleq
    \begin{bmatrix}
        r_i (\*b_i^\T \otimes \eye_3)^\T \ez 
        \pm
        (\*K^\T - \ez \*y_i^\T)(\*b_i \otimes \eye_3) \*{\hat e}_j
        \\
        r_i \ez \pm (\*K^\T - \ez \*y_i^\T)\*{\hat e}_j
    \end{bmatrix}\!,
\end{equation}
for $j=1,2$. There are four backprojection constraints ($\pm$ for $j=1,2$). In linear form,
\begin{equation}
    \eqref{eq:bpinf}_i \iff \left(\*d_{ij}^\pm\right)^\T 
    \begin{bmatrix}
        \*r \\ \*t
    \end{bmatrix}
    \leq 0.
\end{equation}

These are $4N$ constraints and may be written in matrix form as~\eqref{eq:chiralitymatrix} to give the remaining $\*A_i$ matrices.

\section{Sum-of-Squares S-lemma}
\label{appendix:sosslem}

\subsection{Example SOS Inequality}
The SOS inequality is a convenient shorthand for a linear matrix inequality (LMI) arising from polynomials. To illustrate, we rewrite \eqref{eq:sosslem} for a second-order relaxation ($\kappa=1$) and $\*x = [1, x_2]^\T$. Assume without loss of generality that there are no equality constraints ($M=0$) and let $\lambda_i(\*x) = \*x^\T \*A \*x$ for $\*A\succeq 0$. Eq. \eqref{eq:sosslem} is:
\begin{equation}
    \label{eq:start}
    \*x^\T \*W \*x \preceq_{sos} \sum_{i=1}^N \*x^\T \*A_i \*x \*x^\T \*Y_i \*x.
\end{equation}

The right hand side is a fourth-order polynomial. To convert to an LMI we write each term in quadratic form with $[\*x]_2\triangleq [1, x_2, x_2^2]^\T\in\RR^{3}$. Consider $i=1$ and let:
\begin{equation}
    \*A_1 \triangleq \begin{bmatrix}
        a & b \\ b & c
    \end{bmatrix}\!.
\end{equation}

The polynomial $\*x^\T \*A_1 \*x$ is the sum of $3$ terms:
\begin{equation}
    \*x^\T \*A_i \*x = a + 2bx_2 + cx_2^2
\end{equation}

Thus, the product $\*x^\T \*A_1 \*x \*x^\T \*Y_1 \*x$ can be written as $[\*x]_2^\T \*{\hat Y}_1 [\*x]_2$, where $\*{\hat Y}_1$ is:
\begin{equation}
    \resizebox{0.9\columnwidth}{!}
    {$\displaystyle
    \left[
    \begin{array}{@{}c c@{}}
        a\*Y_1 & 
        \begin{matrix}
        0 \\ 0
        \end{matrix}
        \\
        \begin{matrix}
        0 & 0 \\
        \end{matrix} & 0
    \end{array}
    \right]
    +
    \left[
    \begin{array}{@{}c c@{}}
        \begin{matrix}
        0 \\ 0
        \end{matrix}
        & b\*Y_1
        \\
        0 &
        \begin{matrix}
        0 & 0
        \end{matrix}
    \end{array}
    \right]
    +
    \left[
    \begin{array}{@{}c c@{}}
        \begin{matrix}
        0 & 0
        \end{matrix}
        & 0
        \\
        b\*Y_1 & 
        \begin{matrix}
            0 \\ 0
        \end{matrix}
    \end{array}
    \right]
    +
    \left[
    \begin{array}{@{}c c@{}}
        0 &
        \begin{matrix}
        0 & 0
        \end{matrix}
        \\
        \begin{matrix}
            0 \\ 0
        \end{matrix}
        & c\*Y_1
    \end{array}
    \right]\!.
    $}
\end{equation}

Similarly, $\*x^\T \*W \*x$ can be written as $[\*x]_2^\T\*{\hat W}[\*x]_2$ where:
\begin{equation}
    \*{\hat W} \triangleq \begin{bmatrix} \*W & \*0\\ \*0 & 0 \end{bmatrix}\!.
\end{equation}

Repeating for all $i$, the matrix inequality form of \eqref{eq:start} is:
\begin{equation}
    \*{\hat W} \preceq \sum_{i=1}^N \*{\hat Y}_i,
\end{equation}
subject to $\*A_i \succeq 0$ for all $i$. There are automatic tools including TSSOS~\cite{Wang20arXiv-cs-tssos} which perform this calculation.

\subsection{Proof of \prettyref{thm:convergence}: Hierarchy Convergence}
Let $p^\star$ denote the optimal volume of the true minimum volume ellipsoid bounding \eqref{eq:purse_qcqp}. Let $d^\star_\kappa$ be the optimal ellipsoid volume from the $(\kappa+1)$-order dual given in eq. \eqref{eq:conformal_ellipsoid_higherorder}. \prettyref{prop:slemma_sos} guarantees $d^\star_\kappa \leq p^\star$ for all $\kappa$. Our proof will perturb $p^\star$ slightly and show it lower bounds $d^\star_\kappa$ for sufficiently large $\kappa$. 

We use the following positivestellensatz which holds whenever the set \eqref{eq:purse_qcqp} is compact.
\begin{theorem}[\cite{Nie05optim-MinimalEnclosingEllipsoid}]
    \label{thm:putinar}
    Define a compact set $\mathcal{P}\triangleq \{\*x\in\RR^n : h_i(\*x) \leq 0,\ i\in[N]\}$ with polynomials $h_i\in\RR[\*x]$. Any polynomial $p(\*x) < 0$ for $\*x\in\mathcal{P}$ satisfies:
    \begin{equation}
        p(\*x) \preceq_{sos} \sum_{i=1}^N \lambda_i(\*x) h_i(\*x),
    \end{equation}
    where each $\lambda_i(\*x)$ is an SOS polynomial.
\end{theorem}

\prettyref{thm:putinar} is a powerful result connecting sum-of-squares polynomials with strictly positive polynomials. By replacing each equality with two inequalities, \eqref{eq:purse_qcqp} is composed of $5N+30$ polynomial inequality constraints. Thus, we can apply the theorem to any polynomial strictly negative on \eqref{eq:purse_qcqp}.

We'll choose a perturbation of the optimal ellipsoid. Let $g_{\*H}(\*x)\leq 0$ be the ellipsoid \eqref{eq:ellipse_problem} defined by $\*H$. Let the primal optimal ellipsoid matrix be $\*H^\star$. Perturb $\*H^\star$ to a slightly larger ellipsoid by $\*{\tilde H} \triangleq \*H^\star - \epsilon \eye_{12}$ with $\epsilon > 0$. We have $\log\det(\*{\tilde H}) = p^\star - o(\epsilon)$. The perturbation implies $g_{\*{\tilde H}}(\*x) < 0$ on the set \eqref{eq:purse_qcqp}. By \prettyref{thm:putinar}:
\begin{equation}
    g_{\*{\tilde H}}(\*x) \preceq_{sos} \sum_{i=1}^N \lambda_i(\*x) h_i(\*x),
\end{equation}
for some SOS polynomials $\lambda_i(\*x)$, $i\in[5N+30]$.

This is exactly the constraint enforced by the dual problem, with the distinction that $\lambda_i(\*x)$ may be a polynomial \emph{of any order}. Thus, the ellipsoid $\*{\tilde H}$ is feasible for \eqref{eq:conformal_ellipsoid_higherorder} if $\kappa$ is large enough. In practice we must restrict the polynomial order for a computationally tractable problem. However, in the limit of unbounded $\kappa$, the following holds for all $\epsilon > 0$:
\begin{equation}
    p^\star - o(\epsilon) \leq \lim_{\kappa\rightarrow\infty} d^\star_\kappa \leq p^\star.
\end{equation}

Thus, the dual hierarchy converges to the primal. $\qed$

\section{Faster Uncertainty Bounds via Quaternions}
\label{appendix:quaternions}
We derive the backprojection and chirality constraints for the quaternion representation of~\eqref{eq:purse}. This representation reduces the number of variables, significantly improving computation for higher-order relaxations.

\subsection{Quaternion Arithmetic}
The unit quaternion representation of a rotation about axis $\'\omega\in\mathbb{S}^2$ by angle $\theta$ is:
\begin{equation}
    \label{eq:quataxang}
    \*q = \begin{bmatrix}
        \cos(\theta/2)\\
        \'\omega\sin(\theta/2)
    \end{bmatrix}\!,
\end{equation}
where $q_1$ is the scalar part by convention. Unit quaternions have double coverage of the rotation space ($\*q$ and $-\*q$ represent the same rotation). Negating just the vector part gives $\*q^{-1}$. Applying a rotation to a vector requires quaternion algebra. To rotate a point $\*y\in\RR^3$:
\begin{equation}
    \label{eq:quatprodrot}
    \*q\circ
    \begin{bmatrix}
        0\\\*y
    \end{bmatrix}
    \circ\*q^{-1}
     =
     \begin{bmatrix}
        0\\
        \*R\*y
     \end{bmatrix}\!,
\end{equation}
where $\circ$ denotes the \emph{quaternion product} and $\*R\in\SO3$ is the rotation matrix corresponding to the quaternion $\*q$. It also composes rotations: $\*q_1\circ\*q_2 = \*R_1\*R_2$.

Quaternion products can be written as matrix-vector products. Given $\*a\in\RR^4$ and $\*b\in\RR^4$,
\begin{equation}
    \label{eq:quatmath}
    \*a \circ \*b = \*\Omega_1(\*a)\*b = \*\Omega_2(\*b)\*a.
\end{equation}
which defines the following product matrices~\cite{Yang19iccv-QUASAR}:
\begin{equation}
    \label{eq:omegas}
    \resizebox{0.9\columnwidth}{!}
    {$\displaystyle
    \*\Omega_1(\*a) \triangleq
    \begin{bmatrix}
        a_1 & -a_2 & -a_3 & -a_4\\
        a_2 & a_1 & -a_4 & a_3 \\
        a_3 & a_4 & a_1 & -a_2\\
        a_4 & -a_3 & a_2 & a_1
    \end{bmatrix}
    \!\text{, }
    \*\Omega_2(\*a) \triangleq
    \begin{bmatrix}
        a_1 & -a_2 & -a_3 & -a_4\\
        a_2 & a_1 & a_4 & -a_3 \\
        a_3 & -a_4 & a_1 & a_2\\
        a_4 & a_3 & -a_2 & a_1
    \end{bmatrix}\!.
    $}
\end{equation}

When we apply $\*\Omega_1$ or $\*\Omega_2$ to vectors $\*c\in\RR^3$ we implicitly homogenize with a leading $0$, making the matrices skew symmetric. These matrices allow converting between the matrix and quaternion representations of rotation with only vector products.
\begin{lemma}
    Let $\*x,\*y\in\RR^3$ and let $\*q\in\mathbb{S}^3$ and $\*R\in\SO{3}$ be the quaternion and matrix representations of the rotation $(\'\omega, \theta)$. Then,
    \begin{equation}
        \*x^\T\*R\*y
        = 
        -\*q^\T \*\Omega_1({\*x})\*\Omega_2({\*y})\*q.
    \end{equation}
\end{lemma}
\begin{proof}
    From~\eqref{eq:quatprodrot} and~\eqref{eq:quatmath},
    \begin{equation}
        \begin{bmatrix} 0 \\ \*R\*y\end{bmatrix}
        = \*q \circ \*\Omega_2(\*q^{-1})\begin{bmatrix}0 \\ \*y\end{bmatrix}
        = \*\Omega_1(\*q)\*\Omega_2(\*q)^\T\begin{bmatrix}0 \\ \*y\end{bmatrix},
    \end{equation}
    where we use the identity $\*\Omega_2(\*q^{-1}) = \*\Omega_2(\*q)^\T$. Taking the inner product with $\*x$, we arrive at the identity:
    \begin{multline}
        \*x^\T\*R\*y =  \left(\*\Omega_2(\*q)\begin{bmatrix}0 \\ \*x\end{bmatrix}\right)^\T\left(\*\Omega_1(\*q)\begin{bmatrix}0 \\ \*y\end{bmatrix}\right)\\
    = \*q^\T \*\Omega_1(\*x)^\T\*\Omega_2(\*y)\*q.
    \end{multline}
    The property $\*\Omega_1(\*q)\*\Omega_2(\*q)^\T = \*\Omega_2(\*q)^\T\*\Omega_1(\*q)$ may be checked by substitution. Applying the skew symmetric property yields the lemma.
\end{proof}

\subsection{Quaternion Uncertainty Bounds}
Notice that the chirality constraints~\eqref{eq:frontofcamera} and backprojection constraints~\eqref{eq:bpinf} are linear in the rotation matrix $\*R$. For computational efficiency, we can rewrite these constraints in terms of the corresponding unit quaternion $\*q$. The chirality constraint is:
\begin{equation}
    \label{eq:foc_quat}
    \eqref{eq:frontofcamera} \iff 
    -\*q^\T \'\Omega_1(\*{\hat e}_3) \'\Omega_2(\*b_i) \*q
    + \*{\hat e}_3^\T\*t \geq 0,
\end{equation}
and the backprojection constraints are:
\begin{multline}
    \label{eq:bpinf_quat}
    \eqref{eq:bpinf}\! \iff\!
    \*q^\T \'\Omega_1\!\left(
        r_i\ez \pm (\*K - \*y_i\ez^\T)^\T\*{\hat e}_j
    \right)
    \'\Omega_2(\*b_i) \*q \\
    \leq 
    \left(
        r_i\ez \pm (\*K - \*y_i\ez^\T)^\T\*{\hat e}_j
    \right)^\T\! \*t.
\end{multline}
Note there are $4$ backprojection constraints for each keypoint $i$ ($\pm$ for $j=1,2$).

To enforce that $\*q$ represents a rigid rotation we simply require it have unit $2$-norm: $\*q^\T\*q=1$. Additionally, to handle double coverage we add the additional constraint $\*q^\T\*{\bar q} > 0$ to force the quaternion into the same hemisphere as the rotation estimate $\*{\bar q}$.

We now write the quadratic form of~\eqref{eq:purse} in terms of unit quaternions. Let $\*x \triangleq [1, \*q, \*t]$. In quadratic form, the pose uncertainty set is
\begin{equation}
    \label{eq:purse_qcqp_quat}
    \left\{
        \begin{array}{cc}
            \*x\in\RR^{8}\\
            x_{1} = 1
        \end{array}
        \left|\vphantom{\sum_{i}^N}\right.
        \begin{array}{rlc}
            \*x^\T\*Q\*x\hspace{-6pt} &= 0\\
            \*x^\T\*A_i\*x\hspace{-6pt} &\leq 0,\:\: i\in[5N+1]
        \end{array}
    \right\}\!.
\end{equation}
In the above, $\*Q$ enforces the unit quaternion constraint, while the inequality matrices $\*A_i$ enforce the chirality, backprojection, and double coverage constraints. Bounding this set with an ellipse is identical to~\prettyref{sec:uncertainty}.

Crucially, the quaternion representation enables substantial computational speedup by reasoning over a smaller $8\times 8$ lifted matrix instead of a $13\times 13$ matrix. In practice, the first-order relaxation fails to solve but the second-order relaxation is much tighter than the first-order in matrix form and solves in about $500\ \mathrm{ms}$.

\subsection{Angular Ellipse Marginalization}
The joint quaternion-translation ellipse matrix $\*H\in\mathcal{S}^8$ implies a translation and axis-angle ellipse similar to~\prettyref{sec:uncertainty_bounds}. Marginalizing to a translation-only ellipse requires projection onto the last three coordinates as before. The quaternion representation allows a slightly different axis-angle representation which eliminates the need for a bounded angular deviation.

First, we project $\*H$ to its first four coordinates to obtain an ellipse in $\*q$. Let $\*P_q \triangleq \begin{bmatrix} \eye_4 & \*0_{4\times 3}\end{bmatrix}$ and $\*H_q \triangleq \left(\*P_q \*H^{-1} \*P_q^\T\right)^{-1}$. Then the joint ellipse implies:
\begin{equation}
    (\*q - \*{\bar q})^\T\*H_q(\*q - \*{\bar q}) \leq 1.
\end{equation}

Let $\*q_\theta$ perturb $\*q$ by angle $\theta$ about axis $\'\omega\in\mathbb{S}^2$. We can rewrite $\*q = \*q_\theta \circ \*{\bar q}$. Thus,
\begin{equation}
    \label{eq:qqbar}
    \*q - \*{\bar q} = \left(\*q_\theta - \begin{bmatrix}
        1 \\ \*0
    \end{bmatrix}\right)\circ\*{\bar q} = 
    \begin{bmatrix}
        \cos(\theta/2) - 1\\
        \'\omega\sin(\theta/2)
    \end{bmatrix}\circ\*{\bar q}.
\end{equation}

We choose $\'\omega\sin(\theta/2)$ as our axis-angle representation. Note that we no longer require $\theta\leq90^\circ$ to uniquely recover the axis and angle from this representation. From~\eqref{eq:qqbar} and the identity~\eqref{eq:quatmath},
\begin{equation}
    \*H_\theta \triangleq \left(\*P_\theta
    \left[
        \*\Omega_2(\*{\bar q})^\T \*H_q \*\Omega_2(\*{\bar q})
    \right]^{-1}
    \*P_\theta^\T\right)^{-1},
\end{equation}
where $\*P_\theta$ simply projects onto the last $3$ coordinates.

The joint ellipse implies the following ellipsoidal constraint, completing the marginalization.
\begin{equation}
    (\'\omega\sin(\theta/2))^\T\*H_\theta(\'\omega\sin(\theta/2)) \leq 1.
\end{equation}

\section{Additional Experimental Results}
\label{appendix:conformal_exp}
We provide additional results regarding pose estimation (\prettyref{appendix:poseest}), Bayesian uncertainty (\prettyref{appendix:bayesian}), ellipsoid bounds at confidence $\alpha=0.4$ (\prettyref{appendix:vols}), and several ablations of SLUE (\prettyref{appendix:ablations}).

\subsection{Pose Estimation}
\label{appendix:poseest}
In our experiments, we choose perspective-n-point (PnP)~\cite{Terzakis20eccv-sqpnp} for pose estimation, which gives the center of ellipsoidal bound. Notably, our PnP formulation uses a Gaussian noise assumption informed by the keypoint bounds. We do not find a significant difference between this estimate and random sample averaging (RANSAG) \cite{Yang23arxiv-ransag}, a heuristic which averages perspective-3-point solutions sampled from the keypoint bounds. We first present our PnP formulation and then explicitly compare the performance of the two estimators. The effectiveness of PnP highlights the flexibility of uncertainty quantification. SLUE generates accurate uncertainty sets regardless of the estimator priors.

\begin{table}
    \centering
    \caption{Pose 2D Projection Error (\%)}
    \label{tab:poses}
    \begin{tabular}{rcccc}
        \toprule
        & \multicolumn{2}{c}{$\alpha=0.1$} & \multicolumn{2}{c}{$\alpha=0.4$}\\ 
        \cmidrule(lr){2-3}\cmidrule(lr){4-5}
        & PnP & RANSAG & PnP & RANSAG\\
        \midrule
        \textsf{CAST}     & 17.27 & 22.65 & 58.43 & 54.94 \\
        \textsf{LM-O}     & 72.50 & 72.53 & 72.32 & 72.20 \\
        \textsf{YCB-V}    & 51.75 & 45.12 & 53.81 & 50.77 \\
        \bottomrule
    \end{tabular}
\end{table}

\textbf{PnP Formulation.}
We solve the backprojection form of PnP using a second-order semidefinite relaxation. Assuming $\'\eps_i\sim\mathcal{N}(0, \sigma_i\eye_2)$ in \eqref{eq:conformal_meas}, the maximum likelihood estimator for object pose is:

\begin{equation}
    \label{eq:pnp}
    \min_{\substack{\*R\in\SO3, \\\*t\in\RR^3}} \sum_{i=1}^N
    \frac{1}{\sigma_i^2}
    \left\|
        \*{\hat{e}}_3\cdot(\*R\*b_i + \*t)\*y_i - \*K(\*R\*b_i + \*t)
    \right\|^2\!.
\end{equation}

Note that the rotation constraints $\*R\in\SO3$ may be written as $15$ quadratic equalities as in \prettyref{appendix:roteq}. We simply set the variance to the keypoint bound: $r_i = \sigma_i$. We solve \eqref{eq:pnp} using a second-order semidefinite relaxation \cite{Lasserre01siopt-LasserreHierarchy}. This relaxation is empirically tight for all examples, meaning it always returns the maximum likelihood pose estimate.

\textbf{Comparison with RANSAG.} 
We compare Gaussian PnP against RANSAG~\cite{Yang23arxiv-ransag}. For a fair comparison we use the $2$-norm bounds for both methods. \prettyref{tab:poses} computes the \emph{five-pixel 2D projection error} metric~\cite{Brachmann14cvpr} for each dataset. Neither method clearly outperforms the other, and both suffer when keypoint bounds are conservative, which is particularly acute in \textsf{\small CAST}.

\subsection{Comparison With Bayesian Methods}
\label{appendix:bayesian}
The Gaussian formulation for PnP suggests a simpler Bayesian approach to uncertainty quantification. In Bayesian uncertainty quantification, we assume some distribution on keypoint noise and a prior on object pose~\cite{Barfoot17book}. The posterior distribution, obtained via Bayes rule, is a distributional model of uncertainty. Contrary to SLUE, this model comes with no high-probability guarantees like \prettyref{prop:pus} unless the distributional assumptions are exactly satisfied. In this section we show a popular Bayesian approach~\cite{Barfoot17book}, which assumes Gaussian keypoint noise and linearizes the rotation matrix, does not give an accurate uncertainty measure.

First, we rewrite \eqref{eq:pnp} as $\min_{\*R, \*t} \sum_{i=1}^N \|\*q_i\|^2$, where
\begin{equation}
    \*q_i \triangleq \frac1{\sigma_i} \underbrace{(\eye_3 - \*y_i \*{\hat{e}}^\T)\*K}_{\triangleq \*U_i}(\*R \*b_i + \*t).
\end{equation}

\begin{table}[tb]
    \centering
    \caption{Coverage (\%) of Gaussian Model Quantiles}
    \label{tab:bayesian}
    \adjustbox{width=0.9\linewidth}
    {\begin{tabular}{rcccccc}
        \toprule
        Quantile & 0.1 & 0.6 & 0.9 & 0.95 & 0.99 \\
        \midrule
        \textsf{CAST} 
        & 1.5 & 9.9 & 21.2 & 27.5 & 39.3
        \\
        \textsf{LM-O}
        & 59.0 & 87.6 & 93.0 & 94.5 & 96.7
        \\
        \textsf{YCB-V}
        & 75.7 & 91.2 & 93.8 & 94.1 & 94.8
        \\
        \bottomrule
    \end{tabular}}
\end{table}

\begin{figure*}[tb!]
    \centering
\subfloat[Translation uncertainty ellipsoid ($\alpha=0.4$).]{\includegraphics[width=0.35\linewidth]{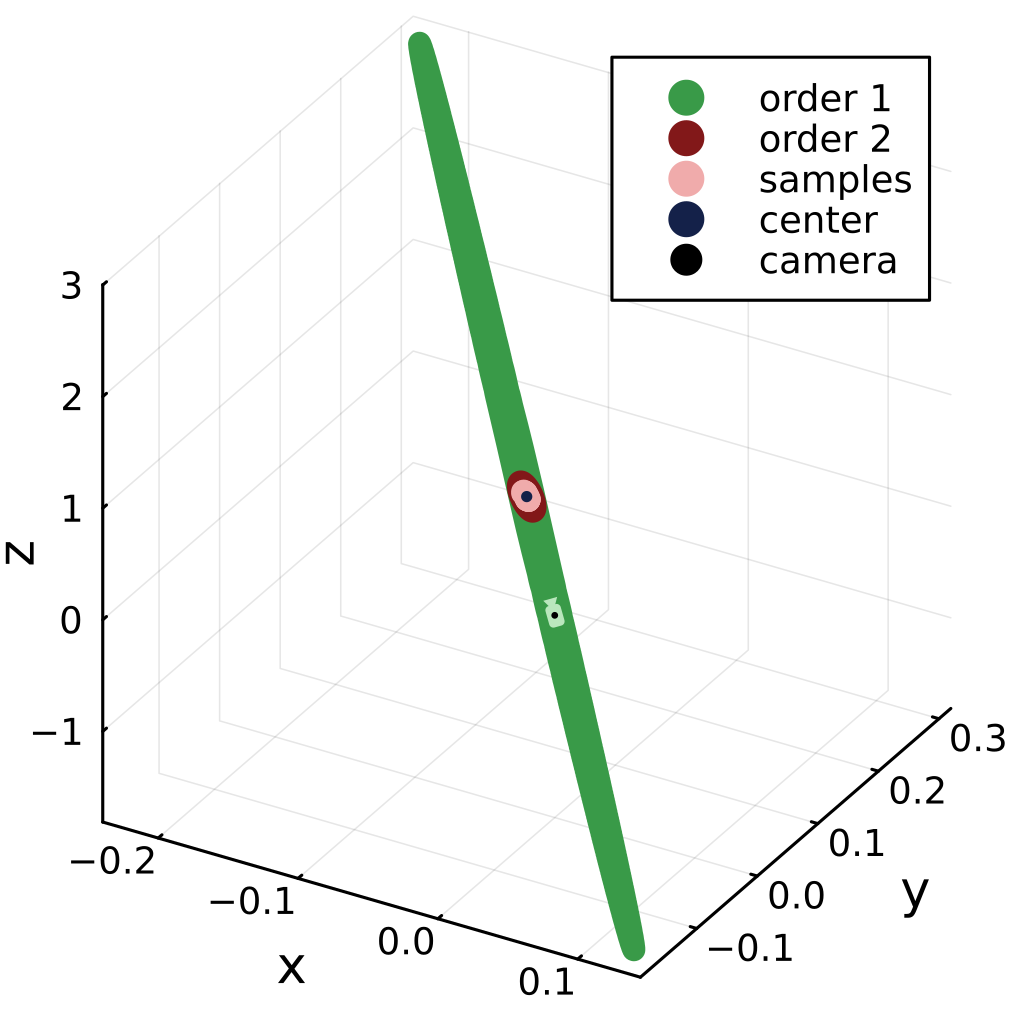}}
    \hspace{0.5cm}
    \subfloat[Orientation uncertainty ellipsoid ($\alpha=0.4$).]{\includegraphics[width=0.4\linewidth]{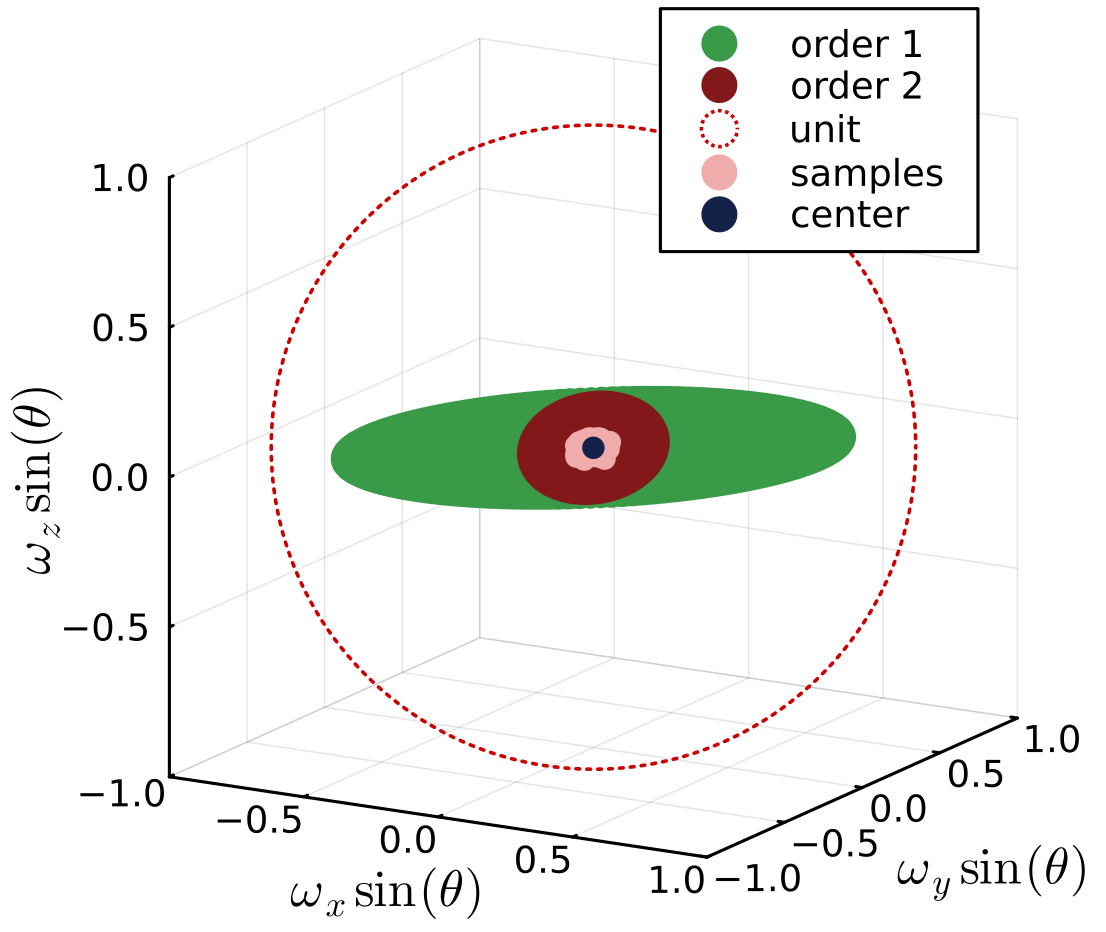}}
    \caption{\textbf{Qualitative bounding sets in rotation and translation space.} Projections of the first and second-order SLUE uncertainty ellipses for the duck object on \textsf{LM-O}. The $\alpha=0.4$ case gives a smaller uncertainty set at the cost of lower confidence. We omit baselines for visual clarity.}
    \label{fig:ellipses_40}
\end{figure*}

\begin{figure*}[htb!]
    \centering
    \subfloat[\textsf{LM-O} uncertainty]{\includegraphics[width=0.47\linewidth]{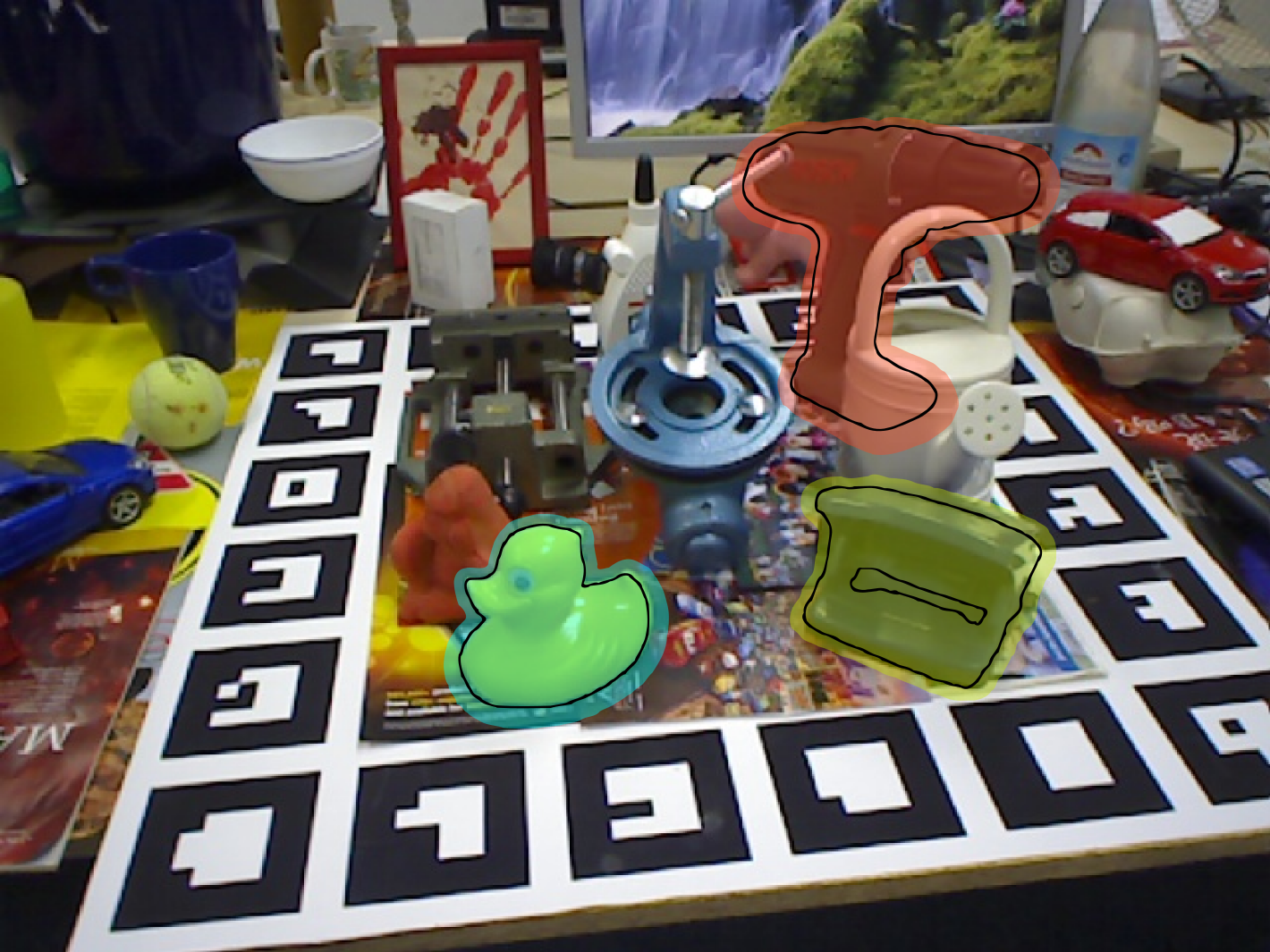}}
    \hspace{0.2cm}
    \subfloat[\textsf{YCB-V} uncertainty]{\includegraphics[width=0.47\linewidth]{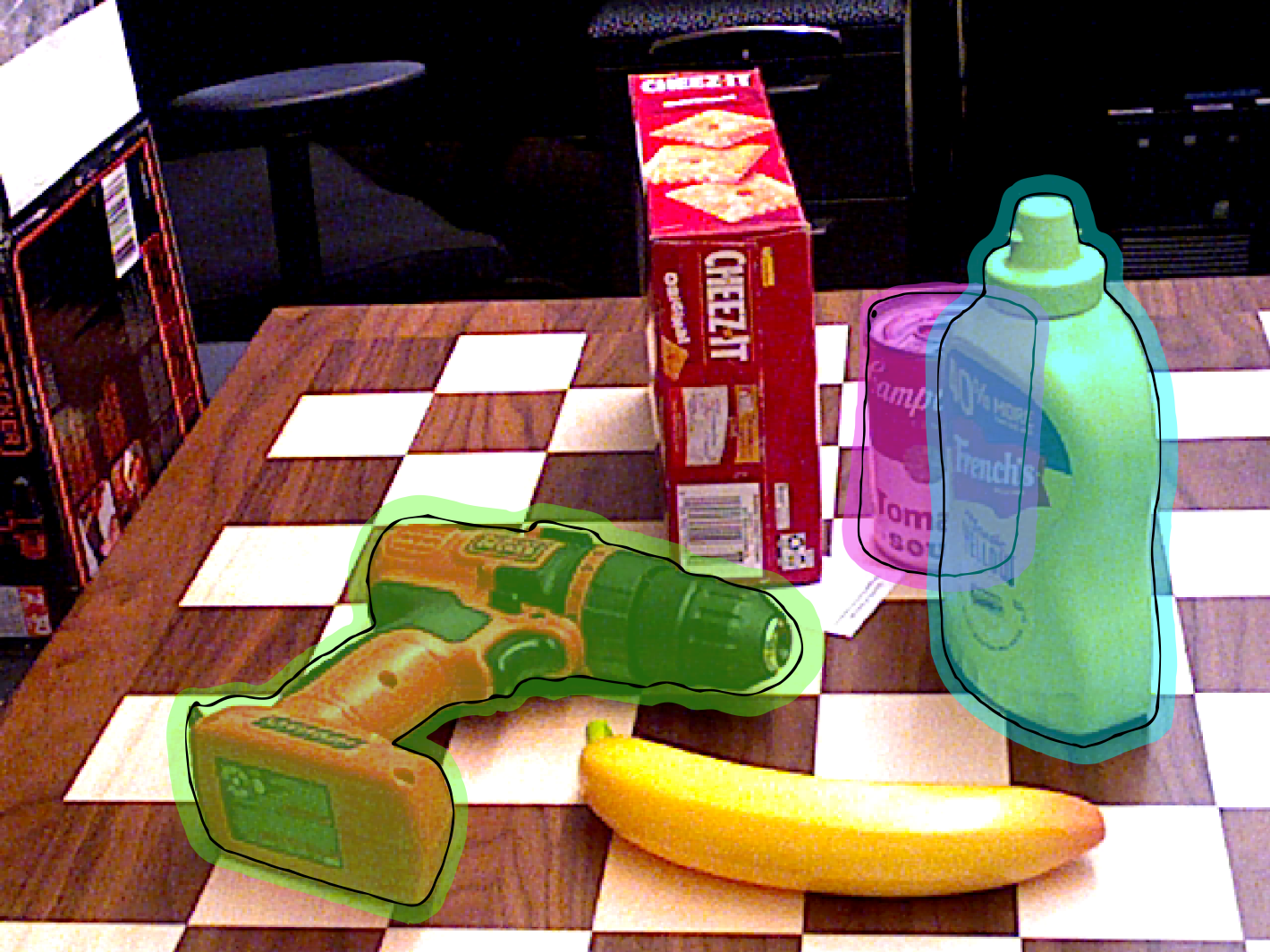}}
    \caption{\textbf{Image-plane projections of ellipsoidal pose uncertainty.} Plots show the set of possible poses in the second-order joint ellipsoidal bound for $\alpha=0.4$, which are significantly smaller than $\alpha=0.1$. We omit \textsf{\small CAST}.}
    \label{fig:qual2}
\end{figure*}

We adopt the tangent space linearization in \cite{Barfoot17book}. That is, $\*R\*b_i = \exp(\widehat{\'\omega\theta})\*{\bar R}\*b_i \approx \*{\bar R}\*b_i - (\widehat{\*{\bar R} \*b_i})$, where the hat operator takes the skew symmetric matrix of a vector as in \eqref{eq:skew}. Thus, the residual is:
\begin{equation}
    \*q_i \approx \frac1{\sigma_i}\*U_i\*{\bar R}\*b_i +
    \underbrace{\frac{1}{\sigma_i}\begin{bmatrix}
        \*U_i (\widehat{\*{\bar R}\*b_i}) &
        \*U_i
    \end{bmatrix}}_{\triangleq \*J_i}
    \begin{bmatrix}
        \'\omega \theta\\
        \*t
    \end{bmatrix}\!.
\end{equation}

Under this linearized residual, the posterior is Gaussian with information matrix $\*H_b \triangleq \sum_{i=1}^N \*J_i^\T \*J_i$. Quantiles of this distribution are proportional to ellipsoids with matrix $\*H_b$ and radius given by quantiles of the $\chi^2$ distribution with $6$ degrees of freedom.

To obtain this model, we made two approximations: we assumed Gaussian measurement noise and linearized the rotation. While it is a convenient approximation, the Gaussian model comes with no guarantees. \prettyref{tab:bayesian} shows the quantiles of this distribution are incorrect in real-world data, and are not guaranteed to lower bound coverage.

\subsection{Qualitative Bounds at $\alpha=0.4$}
\label{appendix:vols}
We provide additional qualitative ellipsoid bounds generated with confidence $\alpha=0.4$, which corresponds to smaller keypoint uncertainty bounds. These smaller bounds generally yield smaller uncertainty sets (which also hold with lower confidence). SLUE in turn produces smaller uncertainty ellipsoids, as shown in the ellipsoids in \prettyref{fig:ellipses_40} and image-plane projections in \prettyref{fig:qual2}. Both figures show the same image frames as \prettyref{sec:results_ellipse}. The figures still show similar ellipsoid trends. In particular, the first-order translation ellipsoid is elongated along the optical axis while the second-order ellipsoid is significantly tighter.

\subsection{Separate Ellipsoid Bounds and Relaxation Order}
\label{appendix:ablations}
Lastly, we provide two ablations to justify our design choices. Specifically, we compare bound volumes when solving for rotation and translation ellipsoids separately, and we compare with results for a third-order relaxation. Both of these design changes improves individual bounds but come at a cost of substantially increased runtime.

\textbf{Separate Rotation and Translation Ellipsoids.} We first compare SLUE with optimizing for rotational and translational ellipsoid volumes separately. This should always produce tighter independent rotational and translational bounds compared with the joint ellipsoid objective. We find optimizing for separate ellipsoids is computationally expensive at higher relaxation orders and offers only a small improvement on bounds. Further, using a joint ellipsoid is more accurate under correlation between translation and orientation uncertainty.

In \eqref{eq:conformal_ellipsoid}, the rotation-only ellipsoid objective corresponds to setting $\*H$ as follows:
\begin{equation}
    \*H \triangleq \begin{bmatrix}
        \*H_r & \*0 \\
        \*0 & \*0_{3\times3}
    \end{bmatrix}\!,
\end{equation}
where $\*H_r \in \mathcal{S}^9$ defines the rotation-only ellipsoid. We must also replace the objective with $\log\det(\*H_r)$; otherwise, it is always $0$.

Similarly, the translation-only ellipsoid objective corresponds to defining $\*H$ as:
\begin{equation}
    \*H \triangleq \begin{bmatrix}
        \*0_{9\times9} & \*0 \\
        \*0 & \*H_t
    \end{bmatrix}\!,
\end{equation}
where $\*H_t \in \mathcal{S}^3$ and we again replace the objective with $\log\det{\*H_t}$, but redefine $\*H$ as above elsewhere.

\prettyref{tab:objective} compares the bound volumes using SLUE (joint), which optimizes for a joint ellipsoid and then projects to rotation and translation ellipsoids, against separately optimizing for rotation and translation ellipsoids directly (split). We also show the total runtime of each approach. With one exception (likely a numerical issue), the split approach produces smaller volume bounds. At first-order, the split approach is also faster. For second-order bounds, however, the split approach is significantly slower because it must solve two large SDPs. In both cases the split approach only offers a minor volume improvement. Compute-limited applications which only require independent bounds should consider using the split formulation.

\begin{table}[tb]
    \centering
    \caption{Ablation: Median Bound Volumes and Runtimes for Split and Joint Ellipsoid Bounds on LM-O}
    \label{tab:objective}
    \adjustbox{width=\linewidth}
    {\begin{tabular}{rcccccc}
        \toprule
        & \multicolumn{2}{c}{Trans. ($\mathrm{m}^3\times10^{-3}$)} & \multicolumn{2}{c}{Ang. ($\mathrm{deg}^3\times10^5$)} & \multicolumn{2}{c}{Time (ms)}\\
        \cmidrule(lr){2-3} \cmidrule(lr){4-5} \cmidrule(lr){6-7}
        & Joint & Split & Joint & Split & Joint & Split \\
        \midrule
        $\kappa=1$
        & 2.7 & 1.4
        & 2.35 & 2.40
        & 37 & 31
        \\
        $\kappa=2$
        & 0.29 & 0.16
        & 4.90 & 3.51
        & 519 & 1056
        \\
        \bottomrule
    \end{tabular}}
\end{table}

\textbf{Relaxation Order.} Although we have only shown first and second-order results, the relaxation hierarchy extends beyond this and is only guaranteed to converge to the minimum volume bound as relaxation order $\kappa\rightarrow\infty$. Implicit in the hierarchy is a computational trade-off between accuracy, which improves at higher order, and runtime, which deteriorates. 

\prettyref{tab:higherorder} gives bound volumes and runtimes up to a third-order relaxation. Across datasets, the third-order relaxation takes a median time of about $7$ seconds, compared to under half a second for the second-order relaxation. The bound volumes do not improve proportionally. Although there is still some improvement in volume between second and third-order, it is significantly smaller than the improvement from first to second-order. In most applications, the precision of a second-order bound will likely be enough.

\begin{table}[tb]
    \centering
    \caption{Median Bound Volumes and Runtimes Across Relaxation Orders}
    \label{tab:higherorder}
    \adjustbox{width=\linewidth}
    {\begin{tabular}{rccccccccc}
        \toprule
        & \multicolumn{3}{c}{Trans. ($\mathrm{m}^3\times10^{-3}$)} & \multicolumn{3}{c}{Ang. ($\mathrm{deg}^3\times10^5$)} & \multicolumn{3}{c}{Time (ms)}\\
        \cmidrule(lr){2-4} \cmidrule(lr){5-7} \cmidrule(lr){8-10}
        $\kappa=$ & 1 & 2 & 3 & 1 & 2 & 3 & 1 & 2 & 3 \\
        \midrule
        \textsf{CAST} 
        & 154 & $9.7$ & $6.1$
        & 4.3 & 2.2 & 1.4
        & 35 & 284 & 7044
        \\
        \textsf{LM-O}
        & 2.7 & 0.29 & 0.19
        & 2.3 & 0.49 & 0.37
        & 37.4 & 520 & 7697
        \\
        \textsf{YCB-V}
        & 5.8 & 0.70 & 0.41
        & 4.2 & 1.6 & 1.0
        & 42.9 & 363 & 7145
        \\
        \bottomrule
    \end{tabular}}
\end{table} 

\bibliographystyle{plain}
\scriptsize{
    
}
\normalsize
\vspace{-30pt}

\begin{IEEEbiography}[{\includegraphics[width=1in,height=1in,clip,keepaspectratio]{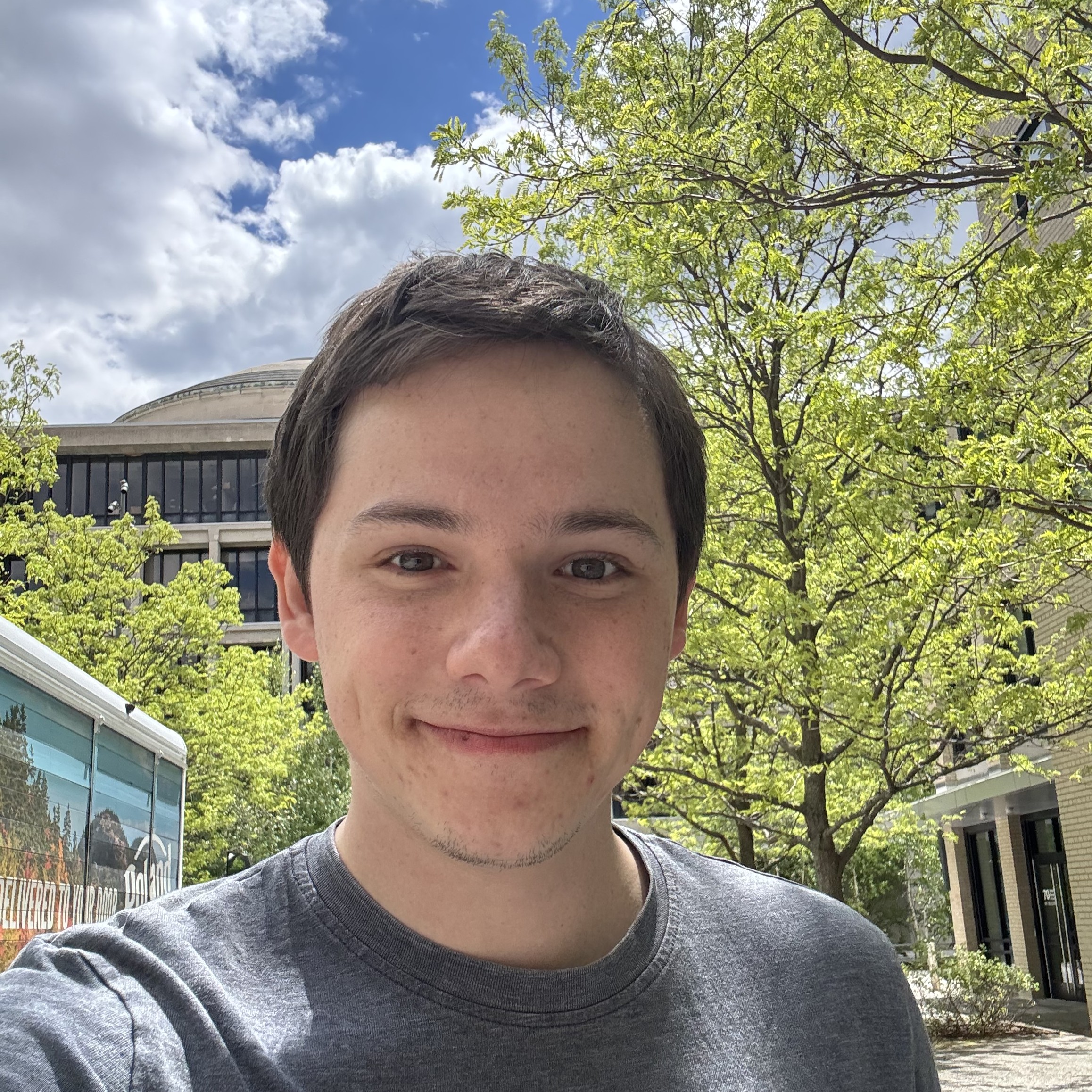}}]{Lorenzo Shaikewitz} (Graduate Student Member, IEEE) received the B.Sc. degree in Mechanical Engineering from the California Institute of Technology, Pasadena, CA, USA in 2023, and the M.Sc. degree in aeronautics and astronautics from the Massachusetts Institute of Technology, Cambridge, MA, USA in 2025.

He is currently a PhD student with the Massachusetts Institute of Technology, Cambridge, MA, USA. His research interests include certificate perception, optimization, and safety guarantees for robot perception and manipulation. His work is supported by an NSF graduate fellowship.
\end{IEEEbiography}
\vspace{-35pt}

\begin{IEEEbiography}[{\includegraphics[width=1in,height=1in,clip,keepaspectratio]{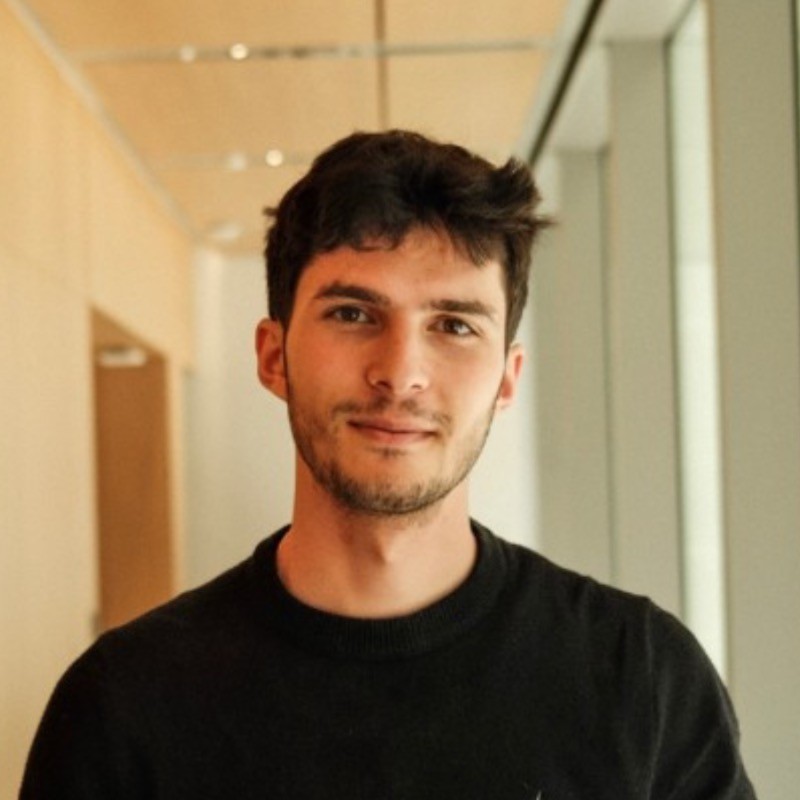}}]{Charis Georgiou} received the B.S. degree in computer science and engineering, and in mathematics from the Massachusetts Institute of Technology, Cambridge, MA, USA, in 2025. 
    
He is currently the Co-Founder and CEO of Minoic Intelligence Inc., Cambridge, MA, USA, a company developing computer vision and autonomous systems for the mining and aggregates industry. His research interests include perception and autonomy for robots operating in extreme, unstructured environments, with a focus on enabling reliable real-time decision-making for systems deployed in harsh industrial field conditions.
\end{IEEEbiography}
\vspace{-35pt}

\begin{IEEEbiography}[{\includegraphics[width=1in,height=1in,clip,keepaspectratio]{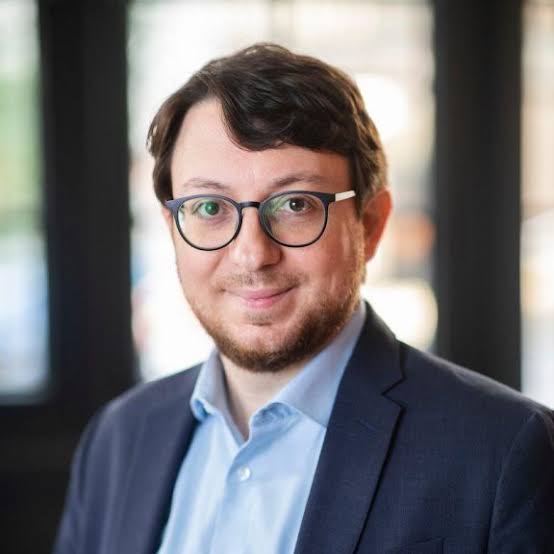}}]{Luca Carlone} (Senior Member, IEEE) received the Ph.D. degree in robotics from the Polytechnic University of Turin, Turin, Italy, in 2012. 
    
He is currently an Associate Professor with the Department of Aeronautics and Astronautics at the Massachusetts Institute of Technology, Cambridge, MA, USA, where he is also a Principal Investigator with the Laboratory for Information and Decision Systems (LIDS). His research interests include estimation, optimization, and learning for robot perception and decision-making. He has received several honors, including the NSF CAREER Award, the RSS Early Career Award, a Sloan Research Fellowship, a National Academy of Sciences Kavli Fellowship, and multiple best paper awards at leading robotics conferences and journals.
\end{IEEEbiography}

\end{document}